\documentclass[sigconf]{acmart}
\usepackage{bm}
\usepackage{multirow}
\usepackage{pifont}
\AtBeginDocument{%
  }

\copyrightyear{2026}
\acmYear{2026}
\setcopyright{cc}
\setcctype{by}
\acmConference[KDD '26]{Proceedings of the 32nd ACM SIGKDD Conference on Knowledge Discovery and Data Mining V.2}{August 09--13, 2026}{Jeju Island, Republic of Korea}
\acmBooktitle{Proceedings of the 32nd ACM SIGKDD Conference on Knowledge Discovery and Data Mining V.2 (KDD '26), August 09--13, 2026, Jeju Island, Republic of Korea}
\acmDOI{10.1145/3770855.3818044}
\acmISBN{979-8-4007-2259-2/2026/08}

\begin{document}

\title{LEFT: Learnable Fusion of Tri-view Tokens for Unsupervised Time Series Anomaly Detection}

\author{Dezheng Wang}
\email{wangdezheng@seu.edu.cn}
\orcid{0000-0002-6449-1043}
\affiliation{
  \institution{School of Automation, Southeast University}
  \city{Nanjing}
  \country{China}
}

\author{Tong Chen}
\email{tong.chen@uq.edu.au}
\orcid{0000-0001-7269-146X}
\affiliation{
  \institution{School of Electrical Engineering and Computer Science, The University of Queensland}
  \city{Brisbane}
  \country{Australia}
}

\author{Guansong Pang}
\email{gspang@smu.edu.sg}
\orcid{0000-0002-9877-2716}
\affiliation{
  \institution{School of Computing and Information Systems, Singapore Management University}
  \country{Singapore}
}

\author{Congyan Chen}
\authornote{Corresponding authors.}
\email{chency@seu.edu.cn}
\orcid{0000-0001-5341-124X}

\author{Shihua Li}
\email{lsh@seu.edu.cn}
\orcid{0000-0001-9044-7137}
\affiliation{
  \institution{School of Automation, Southeast University}
  \city{Nanjing}
  \country{China}
}

\author{Hongzhi Yin}
\authornotemark[1]
\email{h.yin1@uq.edu.au}
\orcid{0000-0003-1395-261X}
\affiliation{
  \institution{School of Electrical Engineering and Computer Science, The University of Queensland}
  \city{Brisbane}
  \country{Australia}
}

\renewcommand{\shortauthors}{Dezheng Wang et al.}

\begin{abstract}
  As a fundamental data mining task, unsupervised time series anomaly detection (TSAD) aims to build a model for identifying abnormal timestamps without assuming the availability of annotations. 
  A key challenge in unsupervised TSAD is that many anomalies are too subtle to exhibit detectable deviation in any single view (e.g., time domain), and instead manifest as inconsistencies across multiple views like time, frequency, and a mixture of resolutions. However, most cross-view methods rely on feature or score fusion and do not enforce analysis–synthesis consistency, meaning the frequency branch is not required to reconstruct the time signal through an inverse transform, and vice versa. In this paper, we present \underline{Le}arnable \underline{F}usion of \underline{T}ri-view Tokens (LEFT), a unified unsupervised TSAD framework that models anomalies as inconsistencies across complementary representations. LEFT learns feature tokens from three views of the same input time series: frequency domain tokens that embed periodicity information, time domain tokens that capture local dynamics, and multi-scale tokens that learn abnormal patterns at varying time series granularities. By learning a set of adaptive Nyquist-constrained spectral filters, the original time series is rescaled into multiple resolutions and then encoded, allowing these multi-scale tokens to complement the extracted frequency and time domain information. When generating the fused representation, we introduce a novel objective that reconstructs fine-grained targets from coarser multi-scale structure, and put forward an innovative time-frequency cycle consistency constraint to explicitly regularize cross-view agreement. As cross-view agreement is explicitly regularized during training, LEFT can adopt lightweight tri-view encoders while maintaining effective coordination among the three views. Experiments on real-world benchmarks show that LEFT achieves the best performance among the compared baselines under the reported evaluation metrics, while using over $6\times$ fewer FLOPs and achieving about $8\times$ faster training. Code is available at \href{https://github.com/DezhengWang/Left.git}{https://github.com/DezhengWang/Left.git}
\end{abstract}

\ccsdesc[500]{Computing methodologies~Temporal reasoning}

\keywords{Time Series Anomaly Detection, Unsupervised Learning, Cross-view Consistency, Tri-view Tokenization, Learnable Filterbank}


\maketitle

\section{Introduction}
Unsupervised TSAD is a core component in industrial systems, as the ability to generate early alarms on abnormal timestamps without annotated data can support safe operation and reduce maintenance and downtime cost~\cite{chen_sequence-aware_2020, zhong_logicgate_2026}. In real deployments, anomalies are heterogeneous: some are short transients, some are gradual drifts, and some mainly appear as changes in periodic structure \cite{wang_cast_2025} and spectral energy~\cite{li_crossad_nodate, wu_catch_nodate}. This heterogeneity renders a detector focusing on a single perspective unreliable, as illustrated in Fig.~\ref{fig_motivation}, and motivates the extraction of multi-view information for unsupervised TSAD \cite{sun_unraveling_2024}.
\begin{figure}[t]
\centering
\includegraphics[width=0.8\columnwidth]{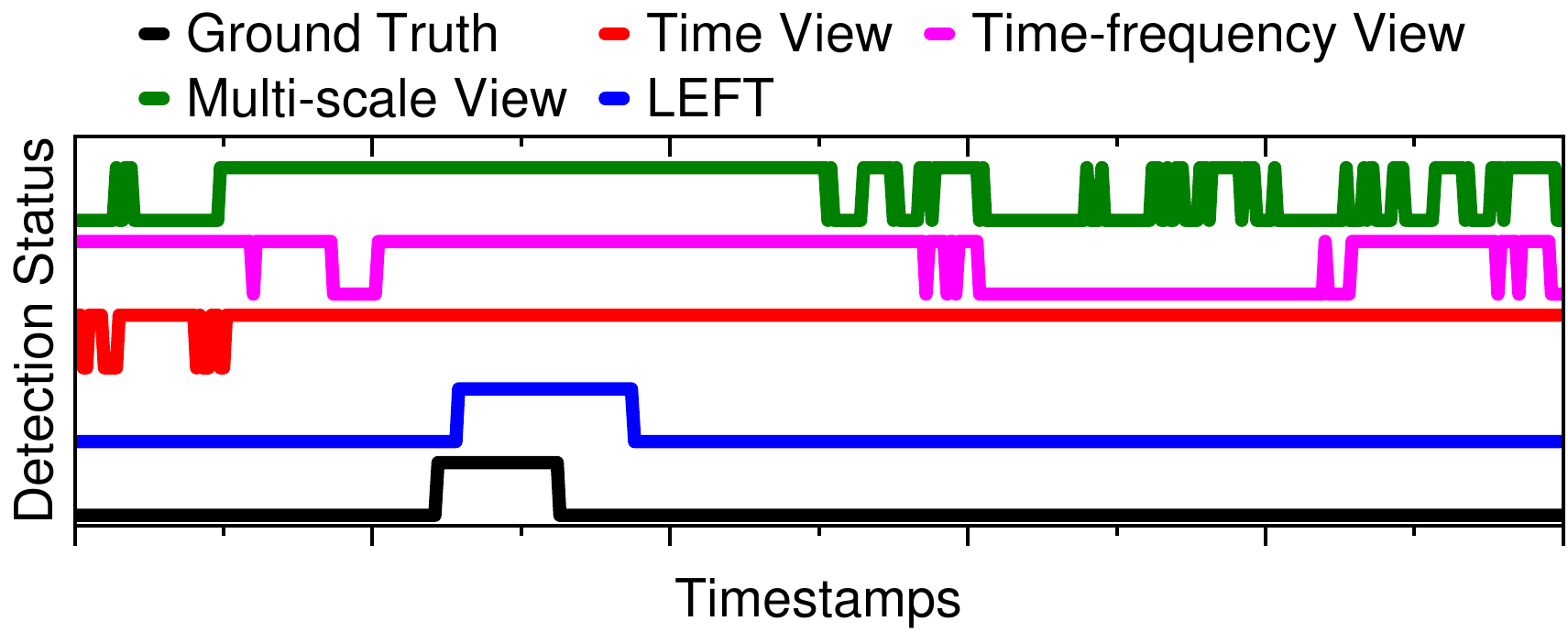}
\caption{An example from SMAP compares predictions from the time-only, time-frequency, multi-scale baseline (CrossAD), and LEFT against the ground truth. High denotes anomalies, and low denotes normal ones.}
\label{fig_motivation}
\end{figure}

From a signal processing perspective, the time and frequency domains are linked by the Fourier transform, which provides a principled connection between these two views~\cite{zhang_improving_2025}. This connection motivates unsupervised TSAD methods to jointly use both views~\cite{nam_breaking_2024}: time domain features capture local dynamics and abrupt deviations~\cite{xia_timeemb_nodate}, while frequency domain features highlight periodic irregularities and band-specific shifts that can be subtle or suppressed by noise in the raw data~\cite{zhong_multi_resolution_2025}. As such, by measuring the disagreement between the representations learned from these two views, an anomaly score can be computed for each timestamp in an unsupervised fashion. 
However, most methods based on time and frequency views still operate on a single sampling resolution, where the model observes only one temporal granularity. Apart from restricting a model's receptive field when learning time series representations, it also prevents the detector from explicitly verifying whether coarse structure (e.g., seasonality) and fine-grained dynamics (e.g., daily change) remain consistent. 
For example, a gradual drift can reshape the coarse trend while slightly detuning short-term rhythms~\cite{zhong_multi_resolution_2025}. In reality, abnormal traces can exhibit different patterns at different sampling granularities, and such anomalies can hardly be detected without inspecting such nuanced coarse-to-fine associations~\cite{li_crossad_nodate, wu_timesnet_nodate}. 
In this regard, introducing multi-scale features at different sampling resolutions as an additional view is a natural enhancement to using only time and frequency features. Intuitively, by jointly modeling time-frequency alignment and cross-resolution consistency, a mutually complementary set of signals can be considered for TSAD. 
However, this brings an additional layer of complexity when learning useful anomaly detection signals from these views with varying semantics, and two significant challenges are to be addressed to maximize the utility of all views for unsupervised TSAD. We outline the challenges below. 

\textit{Challenge 1: How to learn reliable view-specific and scale-specific
representations without supervision?} 
For time and frequency views, most TSAD methods fuse features or scores across both views, or align their feature spaces~\cite{fang_temporal_frequency_2024}. These designs rarely enforce Fourier-based consistency, meaning the predicted spectrum is not required to reconstruct the signal via an inverse transform, and the reconstructed signal is not required to match the spectrum under a forward transform~\cite{wisdom_differentiable_2019}. This leaves the training process prone to learning a biased shortcut that only fits its own view, making cross-domain disagreement a less reliable indicator of true anomalies \cite{fang_temporal_frequency_2024}. For the third view involving multi-scale feature encoding, its performance also hinges on the choice of sampling scales/granularities that are usually tuned via heuristics. However, in different time series data, regimes differ in dominant periods and noise, so the same setting can place the key periodic band at the wrong scale or mix it with noise and miss anomaly-related changes~\cite{wu_catch_nodate}. For example, in some time series with frequent fluctuations, performing straightforward downsampling on time series can fold fast oscillations into low frequencies, so a coarse-scale sequence may show a false flat trend created by aliasing rather than real long-range structure, distorting the learned representations~\cite{nam_breaking_2024}. 

\textit{Challenge 2: How to coordinate cross-view information while preserving useful discrepancies for timestamp-level anomaly scoring?} 
Information fusion across three distinct views is expected to distil complementary information from time, frequency, and multi-scale structure, which enables coherence-based detection on anomalies. However, when generating anomaly scores for each timestamp, existing TSAD methods typically focus on either: (1) the disagreement between features extracted from time and frequency views; or (2) the disagreement across all data resolutions in the multi-scale view. Some works exploit time and frequency views through short-time Fourier transform (STFT) assisted by fusion or reconstruction objectives \cite{fang_temporal_frequency_2024, nam_breaking_2024}, while others emphasize multi-scale features by learning representations at diverse temporal resolutions and aggregating scale-wise outputs \cite{li_crossad_nodate}. While often developed in parallel, few  pipelines jointly model time, frequency, and explicit multi-scale characteristics of the same time series in a unified framework to facilitate effective anomaly scoring on timestamps~\cite{qiu_tab_2025}.

Motivated by these challenges, we propose \textbf{Le}arnable \textbf{F}usion of \textbf{T}ri-view Tokens (\textbf{LEFT}) for unsupervised TSAD, treating anomalies as violations of agreement among time dynamics, time-frequency structure, and multi-scale structure. LEFT builds three complementary views with measurable agreement. It enforces a bidirectional time–frequency cycle, so the time prediction must agree with the input spectrum after analysis and the spectrum prediction reconstructs the input signal after synthesis. This analysis-synthesis coupling goes beyond feature or score fusion, reduces view-specific shortcuts, and makes time–frequency disagreement more reliable anomaly evidence. For multi-scale structure, LEFT uses a Nyquist-constrained learnable filterbank with residual spectral coverage, learning dataset-specific bands without fixed partitions while band-limiting each scale before downsampling to control aliasing and avoid heuristic scale choices. The three token streams interact through selected lightweight fusion rather than dense all-pairs fusion, aligning shared information without washing out view-specific evidence. Training includes a cross-path consistency constraint that aligns the time–frequency cycle with the multi-scale reconstruction, and inference combines multi-scale residuals with cycle-based discrepancy signals to form the final anomaly score.
We summarize our contributions as follows:
\begin{itemize}
    \item We formalize a unified tri-view setting that combines time dynamics, time-frequency structure, and multi-scale structure, where anomalies are exposed as violations of cross-view and cross-resolution consistency.
    \item We propose a Nyquist-constrained learnable filterbank with residual spectral coverage for multi-scale structural tokenization. We introduce bidirectional time-frequency cycle consistency, enforcing analysis-synthesis agreement, and cross-path consistency, aligning the time-frequency pathway with multi-scale reconstruction. We design lightweight token interactions to promote alignment without over-mixing, preserving view-specific evidence for detection.
    \item Extensive experiments show consistent improvements over strong baselines, and demonstrate LEFT improves about 3\% in VUS-ROC and 6\% in VUS-PR, while reducing FLOPs by over 80\% and speeding up training by about $8\times$.
\end{itemize}

\begin{figure*}[t]
\centering
\includegraphics[width=2\columnwidth]{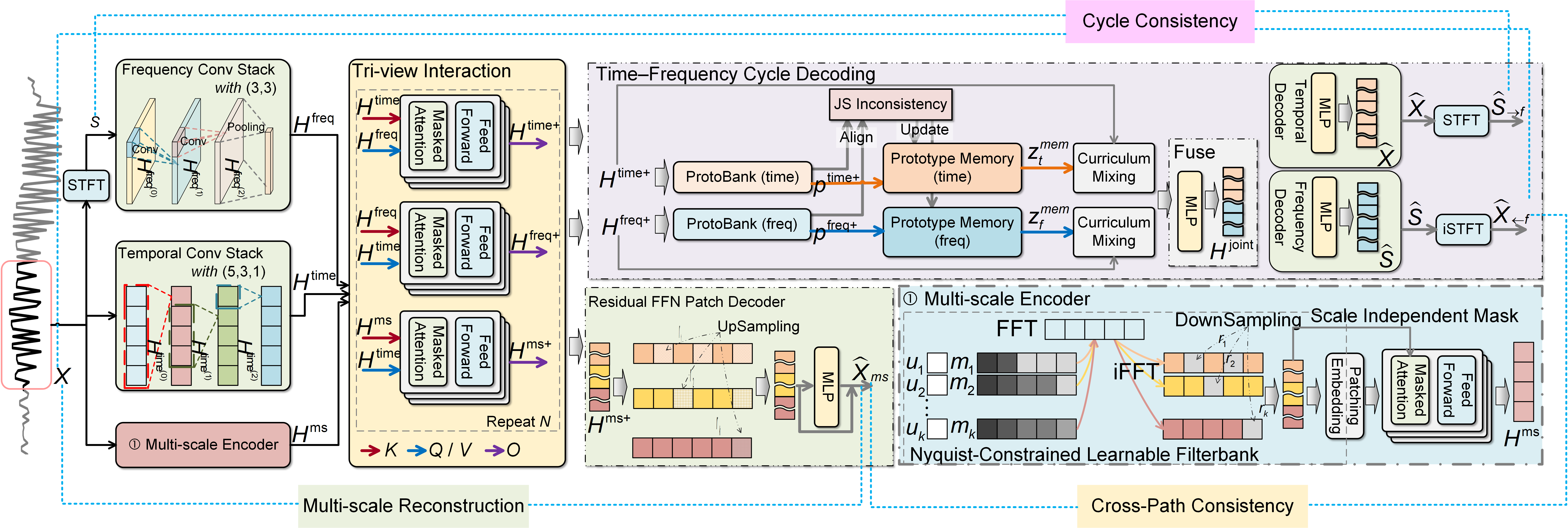}
\caption{The architecture of LEFT.}
\label{fig_overview}
\end{figure*}

\section{Related Work}
\textbf{Unsupervised TSAD}\,\,\,\, Unsupervised TSAD has been studied for decades, from classical pipelines to modern deep models \cite{schmidl_autotsad_2024, li_crossad_nodate}. Classical methods such as LOF \cite{breunig_lof_2000}, DAGMM \cite{zong_deep_2018}, and SVDD \cite{tax_support_2004} are simple and cheap, but they can degrade on industrial data with non-stationary behavior, long-range dependence, and strong variable coupling \cite{behrouz_chimera_2024}. Deep unsupervised TSAD is often grouped into forecasting, reconstruction, and self-supervised or contrastive methods \cite{yang_dcdetector_2023}. Attention-based models are widely used to capture long-range patterns and variable interactions \cite{li_crossad_nodate}. Anomaly Transformer detects anomalies via association inconsistency \cite{xu_anomaly_nodate}, while DCdetector uses discrepancy-aware contrastive learning \cite{yang_dcdetector_2023}.

\textbf{Time–frequency modeling for unsupervised TSAD}\,\,\,\,
Because anomaly evidence can be weak in a single domain, many works incorporate time–frequency information to capture spectral changes such as periodic shifts and band-wise energy redistribution \cite{jin_survey_2024}. Common designs extract Fourier or wavelet features, or adopt frequency-aware architectures \cite{wu_timesnet_nodate}. In many cases, time and frequency are encoded in parallel and coupled through feature alignment or late fusion, which can underuse abnormal time–frequency coupling as a detection signal \cite{nam_breaking_2024}. Empirical studies also suggest that single-scale time domain sampling can miss anomaly signatures in time–frequency analysis \cite{li_crossad_nodate}.

\textbf{Multi-scale modeling for  unsupervised TSAD}\,\,\,\, Multi-scale modeling addresses this issue by representing a series at multiple temporal resolutions. Recent work uses multi-resolution tokenization, patching, or scale mixing \cite{wang_timemixer_nodate}, and TimesNet models multiple period-based views \cite{wu_timesnet_nodate}. In TSAD, multi-scale methods often use per-scale experts with late fusion or shared backbones with implicit scale mixing \cite{shen_time_2021}. Yet anomalies may remain weak at any single resolution and become clearer when coarse structure and fine dynamics stop aligning, which motivates more explicit cross-scale modeling \cite{zhong_multi_resolution_2025, li_crossad_nodate}. At the same time, many multi-scale pipelines focus on time domain representations and do not explicitly integrate frequency domain evidence, even though spectral information can be critical for anomaly detection \cite{wu_catch_nodate, nam_breaking_2024}.

To the best of our knowledge, in unsupervised TSAD, many methods emphasize either time–frequency modeling (e.g.,~\cite{nam_breaking_2024}) or explicit multi-scale structure (e.g.,~\cite{li_crossad_nodate}), while comparatively fewer works integrate both in a unified framework.

\section{Method}
Given a multivariate time-series window $\bm{X}\in\mathbb{R}^{T\times C}$,
the objective is to learn an unsupervised anomaly scoring function that outputs a timestamp-level anomaly map $\mathcal{A}(t)$.

As illustrated in Fig.~\ref{fig_overview}, LEFT constructs three token streams:
$\bm{H}^{\mathrm{time}}$, $\bm{H}^{\mathrm{freq}}$, and $\bm{H}^{\mathrm{ms}}$.
LEFT is based on the premise that normal patterns tend to maintain agreement across domains and resolutions, whereas anomalous patterns are more likely to disturb such agreement. Training is guided by multi-scale reconstruction, cycle consistency, and cross-path consistency, which aligns the cycle pathway with the multi-scale pathway.

\subsection{Tri-view Tokenization}
\label{tokenization}

\subsubsection{Time Tokens}
\label{time_tokens}
Time tokens provide a per-timestamp representation, which preserves local dynamics and abrupt deviations for timestamp-level scoring.
We obtain them with a lightweight 1D convolutional stack $\mathcal{C}^{(\mathrm{1d})}_t$ to provide a stable local inductive bias without resampling:
\begin{equation}
\bm{H}^{\mathrm{time}} = \mathcal{C}^{(\mathrm{1d})}_t(\bm{X}),
\label{eq_time_tokens}
\end{equation}
where $\bm{H}^{\mathrm{time}}\in\mathbb{R}^{T\times D}$ and the kernel sizes of $\mathcal{C}^{(\mathrm{1d})}_t$ uses 1D convolutions with kernel sizes $(5,3,1)$ and corresponding paddings to preserve temporal resolution $T$.

\subsubsection{Frequency Tokens}
\label{freq_tokens}
Frequency tokens capture short-time spectral structure, including changes in periodicity and shifts in band-wise energy that may be hard to see in raw samples. We obtain $\bm{S}$ using a differentiable short-time Fourier transform. This design makes the frequency pathway directly trainable under the cycle consistency objective in Sec.~\ref{cycle}, which encourages spectral consistency with the reconstructed signal. $\bm{S}$ is obtained from $\bm{X}$ as:
\begin{equation}
\bm{S}=W_\theta(\bm{X}),
\label{eq_stft_def}
\end{equation}
where $\bm{S} \in\mathbb{R}^{C\times 2\times F\times T_F}$. $W_\theta(\cdot)$ denotes the differentiable STFT used in the time-frequency branch. In the rest, we use $\mathcal{F}(\cdot)$ and $\mathcal{F}^{-1}(\cdot)$ for the FFT/iFFT used in the multi-scale decomposition. $F$ is the number of frequency bins, and $T_F$ is the number of STFT frames.
Then a frequency encoder $\mathcal{C}^{(\mathrm{2d})}_f$ produces frame-wise tokens:
\begin{equation}
\bm{H}^{\mathrm{freq}} = \mathcal{C}^{(\mathrm{2d})}_f(\bm{S}),
\label{eq_freq_tokens}
\end{equation}
where $\bm{H}^{\mathrm{freq}}\in\mathbb{R}^{T_F\times D}$ and $\mathcal{C}^{(\mathrm{2d})}_f$ treats $(2C)$ as channels, applies two 2D convolution layers with kernel size $(3,3)$, pools over $F$ to obtain frame-wise features, and projects to dimension $D$.

\subsubsection{Multi-scale Structural Tokens via Nyquist-Constrained Learnable Filterbank}
\label{ms_tokens}
To capture multi-scale structure without fixed band partitions, we construct $K$ band-limited downsampled sequences
$\{\bm{X}^{(k)}\}_{k=1}^{K}$ using a Nyquist-constrained learnable filterbank.
Unlike enlarging the receptive field, this construction explicitly produces cross-resolution structure, so coarse-to-fine agreement can be checked during training and inference. Multi-scale tokens are only useful if they remain stable after downsampling.

\begin{lemma}
\label{lem_nyquist_alias}
The learned edges satisfy $0=e_0\le e_1\le\cdots\le e_K$ and $e_k\le c_k$ for all $k$.
Define leakage beyond the cutoff by $\epsilon_k=\sum_{f>c_k}\tilde m_k(f)$.

Let $\tilde{\bm{S}}=\mathcal{F}(\tilde{\bm{X}})$, $\tilde{\bm{S}}^{(k)}(f)=\tilde m_k(f)\tilde{\bm{S}}(f)$, and $\bm{X}^{(k)}=\mathcal{D}_{r_k}\!\left(\mathcal{F}^{-1}\!\left(\tilde{\bm{S}}^{(k)}\right)\right)$,
where $\mathcal{D}_{r_k}(\cdot)$ denotes downsampling by a factor of $r_k$.
Then the aliasing energy after downsampling by $r_k$ satisfies,
\begin{equation}
\begin{split}
\bigl\|\Delta_{r_k}(\bm{X}^{(k)})\bigr\|_2
\ \le\
\vartheta(r_k)\,\|\tilde{\bm{X}}\|_2\,\epsilon_k,\quad
\epsilon_k:=\sum_{f>c_k}\tilde m_k(f),    
\end{split}
\label{eq_alias_bound}
\end{equation}
where $\vartheta(r_k)$ depends only on $r_k$ and FFT conventions.
\end{lemma}

Lemma~\ref{lem_nyquist_alias} (proof in Appendix~\ref{sec_proof_lemma1}) shows that a Nyquist-feasible learnable filterbank provides explicit control over aliasing. After downsampling, the aliasing energy is bounded by a term that scales with the out-of-cutoff leakage $\epsilon_k$, so keeping $\epsilon_k$ small makes the downsampled component $\bm{X}^{(k)}$ a stable approximation of a band-limited signal. This property is useful for uncovering normal and abnormal patterns in the multi-scale view. Under normal operation, most energy stays within the learned cutoff, so coarse tokens preserve consistent long-term structure and support reliable coarse-to-fine reconstruction. When abnormal behavior shifts energy across bands or introduces atypical high-frequency content, $\epsilon_k$ can increase and cross-scale agreement becomes harder to satisfy, which makes inconsistencies more detectable through reconstruction. 

Given a boundary-extended signal $\tilde{\bm{X}}$, let $\tilde{\bm{S}}=\mathcal{F}(\tilde{\bm{X}})$. We then apply $K$ learnable soft masks $\{\tilde m_k(f)\}_{k=1}^K$ to $\tilde{\bm{S}}$:
\begin{equation}
\tilde{\bm{S}}^{(k)}(f)=\tilde m_k(f)\tilde{\bm{S}}(f), \quad k=1,\ldots,K.
\label{eq_band_decomp}
\end{equation}
Let downsampling factors $\{r_k\}_{k=1}^{K}$ be ordered coarse to fine ($r_1\ge\cdots\ge r_K$),
and define normalized Nyquist cutoffs $c_k=\tfrac{1}{2r_k}$.
Band edges are parameterized monotonically with learnable parameters $\{u_k\}$:
\begin{equation}
e_k=e_{k-1}+\bigl(c_k-e_{k-1}\bigr)\sigma(u_k),\quad k=1,\ldots,K,
\label{eq_learned_edges}
\end{equation}
where $e_0=0$ and $\sigma(\cdot)$ is the logistic sigmoid. Given $s_\tau(x)=\sigma\big(x/(\tau+\varepsilon)\big)$ with temperature $\tau>0$, unnormalized band masks are constructed by:
\begin{equation}
m_k(f)=s_\tau(f-e_{k-1})-s_\tau(f-e_k),\quad k=1,\ldots,K,
\label{eq_soft_masks_raw}
\end{equation}
and normalized as:
\begin{equation}
\tilde m_k(f)=\frac{m_k(f)}{\sum_{j=1}^{K}m_j(f)+\varepsilon}.
\label{eq_mask_norm}
\end{equation}
The residual spectral coverage is implemented with a residual mask above the last learned edge.
Within the learned structural bands, the normalized band masks form an approximate partition, i.e.,
\begin{equation}
\sum_{k=1}^{K}\tilde m_k(f)\approx 1
\quad \forall f \ \text{s.t.}\quad
\sum_{k=1}^{K} m_k(f)>0.
\label{eq_mask_partition}
\end{equation}
Each component is transformed back by $\mathcal{F}^{-1}(\cdot)$ and downsampled:
\begin{equation}
\bm{X}^{(k)}=\mathcal{D}_{r_k}\!\left(\mathcal{F}^{-1}\!\left(\tilde{\bm{S}}^{(k)}\right)\right), \quad
T_k=\left\lceil \frac{T}{r_k}\right\rceil,
\label{eq_downsample}
\end{equation}
where $\bm{X}^{(k)}\in\mathbb{R}^{T_k\times C}$.
Each scale is patch-tokenized and projected into $D$ dimensions:
\begin{equation}
\bm{H}^{(k)}=\mathcal{P}\bigl(\bm{X}^{(k)}\bigr), 
\label{eq_ms_patch}
\end{equation}
where $\bm{H}^{(k)}\in\mathbb{R}^{L_k\times D}$. $\mathcal{P}(\cdot)$ denotes the patch tokenization and projection operator that maps each downsampled sequence into token representations.
To preserve scale-specific structural patterns within the multi-scale encoder, we use a scale independent attention mask
$\bm{M}_{\mathrm{blk}}\in\mathbb{R}^{L\times L}$ that is block-diagonal over scales and shared across layers.
This prevents mixing across scales inside $\phi^{\mathrm{ms}}$. Cross-view and cross-resolution exchange is handled later by the tri-view interaction module:
\begin{equation}
\bm{H}^{\mathrm{ms}}
=
\phi^{\mathrm{ms}}\!\left(\mathrm{Concat}\bigl(\bm{H}^{(1)},\ldots,\bm{H}^{(K)}\bigr);\bm{M}_{\mathrm{blk}}\right),
\label{eq_ms_encoder}
\end{equation}
where $\phi^{\mathrm{ms}}(\cdot;\bm{M}_{\mathrm{blk}})$ denotes a masked self-attention encoder.

\subsection{Tri-view Interaction}
\label{interaction}

Given $(\bm{H}^{\mathrm{time}},\bm{H}^{\mathrm{freq}},\bm{H}^{\mathrm{ms}})$, LEFT uses lightweight token-level interaction to share evidence while preserving view-specific information. LEFT avoids dense all-pairs fusion because some view pairs are scale-mismatched. $\bm{H}^{\mathrm{freq}}$ is organized by STFT frames, while $\bm{H}^{\mathrm{ms}}$ comes from downsampled sequences. Direct attention between them would force implicit alignment across scales, which can over-mix features and amplify noise.

LEFT therefore uses selected links. Time and frequency views are connected by the Fourier transform, so $\bm{H}^{\mathrm{time}}$ attends to $\bm{H}^{\mathrm{freq}}$ to inject spectral information into time-aligned representations, and $\bm{H}^{\mathrm{freq}}$ attends back to $\bm{H}^{\mathrm{time}}$ to stay tied to the raw signal. Multi-scale tokens are also derived from the time domain, so $\bm{H}^{\mathrm{time}}$ updates $\bm{H}^{\mathrm{ms}}$. This design also allows frequency evidence to influence the multi-scale pathway indirectly through the updated time tokens $\bm{H}^{\mathrm{time}}$. The \textbf{Fusion Strategy Analysis} further shows that all-pairs fusion or denser fusion does not improve performance in our setting.

The interaction module is a stack of $L_f$ shared fusion blocks:
\begin{equation}
(\bm{H}^{\mathrm{time}+},\bm{H}^{\mathrm{freq}+},\bm{H}^{\mathrm{ms}+})
=
\Phi^{(L_f)}\!\left(\bm{H}^{\mathrm{time}},\bm{H}^{\mathrm{freq}},\bm{H}^{\mathrm{ms}}\right),
\label{eq_triview_fusion_compact}
\end{equation}
where each $\Phi^{(L_f)}$ layer updates the three views via directed cross-attention:
$\bm{H}^{\mathrm{time}}\!=\!\phi^{\text{tf}(L_f)}(\bm{H}^{\mathrm{time}},\bm{H}^{\mathrm{freq}})$,
$\bm{H}^{\mathrm{freq}}\!=\!\phi^{\text{ft}(L_f)}(\bm{H}^{\mathrm{freq}},\bm{H}^{\mathrm{time}})$,
and $\bm{H}^{\mathrm{ms}}\!=\!\phi^{\text{mt}(L_f)}(\bm{H}^{\mathrm{ms}},\bm{H}^{\mathrm{time}})$,
followed by a token-wise residual with normalization.
Finally, this selected fusion is closed by the cross-path constraint $\mathcal{L}_{\mathrm{cons}}$ in Sec.~\ref{cycle}, which ties the multi-scale reconstruction $\hat{\bm{X}}_{\mathrm{ms}}$ to the cycle reconstruction $\hat{\bm{X}}_{\leftarrow f}$, so alignment learned through interaction is made verifiable at the reconstruction level rather than remaining a purely embedding-level agreement.

\subsection{Prototype-based Time–Frequency Cycle Calibration and Memory Mixing}
\label{proto}
To provide a shared reference for matching time and frequency representations and reduce noisy or drifting alignment, we introduce two prototype banks $\bm{P}_t,\bm{P}_f\in\mathbb{R}^{M\times D}$ and map fused tokens to soft prototype assignments:
\begin{equation}
\bm{p}_t=\eta\!\left(\gamma\,\mathrm{cos}\!\left(\bm{H}^{\mathrm{time}+},\bm{P}_t\right)\right),\quad
\bm{p}_f=\eta\!\left(\gamma\,\mathrm{cos}\!\left(\bm{H}^{\mathrm{freq}+},\bm{P}_f\right)\right),  
\label{eq_proto_assign}
\end{equation}
where $\eta$ is Softmax. $\mathrm{cos}(\cdot,\cdot)$ denotes cosine similarity and $\gamma$ is a temperature.
We align $\bm{p}_f$ to length $T$ by interpolation and compute a prototype-assignment discrepancy between the time and frequency views:
\begin{equation}
d_{\mathrm{JS}}(t)=\mathrm{JS}\!\left(\bm{p}_t(t),\,\tilde{\bm{p}}_f(t)\right).
\label{eq_js}
\end{equation}
The discrepancy $d_{\mathrm{JS}}(t)$ is used for prototype calibration during training. Specifically, LEFT uses the averaged JS discrepancy, together with the reconstruction error and assignment confidence, to select reliable samples for EMA-based prototype updates. This design avoids directly optimizing JS divergence while still using cross-view assignment disagreement to maintain stable prototype memories. We also derive an uncertainty gate $g(t)$ from assignment sharpness and entropy. Different from $d_{\mathrm{JS}}(t)$, this gate is directly used as auxiliary evidence in the inference score.

The prototypes further act as a memory of stable patterns, read by time-averaged assignments:
\begin{equation}
\bm{z}_t^{\mathrm{mem}}=\Bigl(\mathrm{Mean}_t(\bm{p}_t)\Bigr)\bm{P}_t,
\quad
\bm{z}_f^{\mathrm{mem}}=\Bigl(\mathrm{Mean}_t(\bm{p}_f)\Bigr)\bm{P}_f.
\label{eq_mem_read}
\end{equation}

Let $\bm{z}_t=\mathrm{Mean}(\bm{H}^{\mathrm{time}+})$ and $\bm{z}_f=\mathrm{Mean}(\bm{H}^{\mathrm{freq}+})$.
We use curriculum memory mixing to stabilize decoding, gradually increasing the memory contribution with $\lambda\in[0,1]$:
\begin{equation}
\bm{z}_t^{+}=(1-\lambda)\bm{z}_t+\lambda\bm{z}_t^{\mathrm{mem}},
\quad
\bm{z}_f^{+}=(1-\lambda)\bm{z}_f+\lambda\bm{z}_f^{\mathrm{mem}}.
\label{eq_mem_mix}
\end{equation}

\subsection{Time-Frequency Cycle Decoding}
\label{cycle}
The time-frequency pathway is trained as a bidirectional closed loop through $W_\theta$ and $W_\theta^{-1}$.
The key purpose is to make the frequency prediction physically checkable: a valid $\hat{\bm{S}}$ should synthesize to a consistent signal, and a reconstructed signal should induce a consistent time-frequency representation.
Without this loop, the frequency branch can be under-constrained and drift to spectra that match coarse energy statistics but are not synthesis-consistent, which weakens cross-domain discrepancy as anomaly evidence.

We fuse time/frequency latents into a shared $\bm{z}=\phi\!\left([\bm{z}_t^{+};\bm{z}_f^{+}]\right)\in\mathbb{R}^{D}$. Then decode into both time and frequency outputs, $\hat{\bm{X}} = D_t(\bm{z})$ and $\hat{\bm{S}} = D_f(\bm{z})$.
Define cycle reconstructions:
\begin{equation}
\hat{\bm{X}}_{\leftarrow f}=W_\theta^{-1}(\hat{\bm{S}}),
\quad
\hat{\bm{S}}_{\rightarrow f}=W_\theta(\hat{\bm{X}}).
\label{eq_cycle_decode_aux}
\end{equation}

Enforcing both directions makes the two outputs mutually verifiable and discourages one-branch shortcut fitting: a spectrum is penalized if it cannot synthesize back to a consistent $\hat{\bm{X}}_{\leftarrow f}$, and a waveform is penalized if it induces an inconsistent $\hat{\bm{S}}_{\rightarrow f}$.

\begin{lemma}
\label{lem_tf_cycle}
Suppose there exist constants $0<A\le B<\infty$ such that for any $\bm{Y}$ for which $W_\theta(\bm{Y})$ is well-defined:
\begin{equation}
\sqrt{A}\,\|\bm{Y}\|_2 \le \|W_\theta(\bm{Y})\|_2 \le \sqrt{B}\,\|\bm{Y}\|_2.
\label{eq_frame_bounds}
\end{equation}

Then the reconstruction errors in the time and time-frequency domains satisfy:
\begin{align}
\|\hat{\bm{X}}-\bm{X}\|_2
&\le \frac{1}{\sqrt{A}}\,\|\hat{\bm{S}}_{\rightarrow f}-\bm{S}\|_2,
\label{eq_time_from_freq}\\
\|\hat{\bm{S}}_{\rightarrow f}-\bm{S}\|_2
&\le \sqrt{B}\,\|\hat{\bm{X}}-\bm{X}\|_2,
\label{eq_freq_from_time}
\end{align}
where $\bm{S}=W_\theta(\bm{X})$ and $\hat{\bm{S}}_{\rightarrow f}=W_\theta(\hat{\bm{X}})$.
\end{lemma}

From Lemma \ref{lem_tf_cycle} (proof is given in Appendix \ref{sec_proof_lemma2}), when the time-frequency transform $W_\theta$ satisfies frame bounds, the reconstruction errors in the time domain and the time-frequency domain bound each other up to constants. Therefore, beyond feature or score level fusion, we introduce the following cycle consistency constraint:
\begin{equation}
\mathcal{L}_{\mathrm{cyc}}
=
l\!\left(\hat{\bm{S}}_{\rightarrow f},\,\bm{S}\right)
+
l\!\left(\hat{\bm{X}}_{\leftarrow f},\,\bm{X}\right).
\label{eq_l_cyc}
\end{equation}

We add a cross-path constraint to explicitly align the multi-scale reconstruction pathway with the time-frequency cycle pathway. This alignment is required for tri-view learning, since LEFT aims to enforce agreement among time, frequency, and multi-scale structure within a single model. The alignment target can be either $\hat{\bm{X}}_{\leftarrow f}$ or $\hat{\bm{X}}$. The cycle loss in Eq.~\ref{eq_l_cyc} aligns both reconstructions to the same reference $\bm{X}$, and it also implies that $\hat{\bm{X}}_{\leftarrow f}$ and $\hat{\bm{X}}$ become close because $\|\hat{\bm{X}}-\hat{\bm{X}}_{\leftarrow f}\|_2 \le \|\hat{\bm{X}}-\bm{X}\|_2 + \|\hat{\bm{X}}_{\leftarrow f}-\bm{X}\|_2$. Therefore, it is sufficient to constrain $\hat{\bm{X}}_{\mathrm{ms}}$ to one of them. We choose $\hat{\bm{X}}_{\leftarrow f}$ and enforce their agreement as follows:
\begin{equation}
\mathcal{L}_{\mathrm{cons}}
=
l\!\left(\hat{\bm{X}}_{\mathrm{ms}},\,\hat{\bm{X}}_{\leftarrow f}\right).
\label{eq_l_cons}
\end{equation}
This ties the multi-scale reconstruction pathway to the cycle pathway, so cross-resolution structure learned by $\hat{\bm{X}}_{\mathrm{ms}}$ is required to agree with the synthesis-consistent reconstruction $\hat{\bm{X}}_{\leftarrow f}$.
The tri-view alignment is enforced at the reconstruction level, and cross-path inconsistencies are unlikely to be caused solely by representation drift, making them useful evidence for anomaly scoring.

\subsection{Multi-scale Reconstruction Head}
\label{ms_recon}
The multi-scale pathway aims to reconstruct Nyquist-feasible structural content across resolutions, so that coarse-to-fine agreement can be explicitly checked.
Starting from fused multi-scale tokens $\bm{H}^{\mathrm{ms}+}$, we perform reconstruction in a coarse-to-fine manner: coarser tokens provide structural context while finer targets enforce detailed consistency.
To retain scale-specific structure while injecting cross-view evidence, we combine the original multi-scale tokens and the interaction-enhanced tokens:
\begin{equation}
\tilde{\bm{H}}^{\mathrm{ms}}=\beta_{\mathrm{res}}\bm{H}^{\mathrm{ms}}+\beta_{\mathrm{int}}\bm{H}^{\mathrm{ms}+},
\label{eq_ms_fuse}
\end{equation}
where $\bm{H}^{\mathrm{ms}}$ preserves scale-pure structural evidence produced by the block-diagonal encoder in Eq.~\ref{eq_ms_encoder}, while $\bm{H}^{\mathrm{ms}+}$ carries alignment signals injected through tri-view interaction.

We then reconstruct multi-scale targets by splitting tokens by scale, upsampling in token space, and decoding patches:
\begin{equation}
\hat{\bm{Y}}
=
D_{\mathrm{ms}}\!\left(\mathrm{Up}\bigl(\mathrm{Split}(\tilde{\bm{H}}^{\mathrm{ms}})\bigr)\right),
\label{eq_ms_decode}
\end{equation}
which yields reconstructions for scales $k=2,\ldots,K$ and the full-rate reconstruction
$\hat{\bm{X}}_{\mathrm{ms}}\in\mathbb{R}^{T\times C}$.
We treat the coarsest scale $k=1$ as structural context and do not supervise it, since directly fitting the coarsest component can encourage overly smooth solutions while contributing limited timestamp-level anomaly localization.

Define the multi-scale target:
\begin{equation}
\bm{Y}=\mathrm{Concat}\bigl(\bm{X}^{(2)},\ldots,\bm{X}^{(K)},\bm{X}\bigr).
\label{eq_ms_target}
\end{equation}

We optimize a segment-aware objective so each scale contributes comparably despite different lengths:
\begin{equation}
\begin{split}
\mathcal{L}_{\mathrm{ms}}
&=
\sum_{k=2}^{K}\omega_k\,l\!\left(\hat{\bm{X}}^{(k)}_{\mathrm{ms}},\bm{X}^{(k)}\right)
+\omega_{\mathrm{full}}\,l\!\left(\hat{\bm{X}}_{\mathrm{ms}},\bm{X}\right),
\\
&\text{s.t.} \quad \sum_{k=2}^{K}\omega_k+\omega_{\mathrm{full}}=1,    
\end{split}
\label{eq_l_ms}
\end{equation}
where $l(\cdot,\cdot)$ denotes SmoothL1.

\subsection{Training Objective}
\label{objective}
LEFT is trained to make agreement an inherent property: the multi-scale pathway should reconstruct Nyquist-feasible structure across resolutions, the time-frequency pathway should satisfy bidirectional analysis-synthesis consistency, and the two pathways should agree on the full-rate reconstruction.
Accordingly, the overall training loss combines (i) multi-scale reconstruction, (ii) cycle consistency, and (iii) cross-path consistency on the time domain:
\begin{equation}
\mathcal{L}
=
\lambda_{\mathrm{ms}}\mathcal{L}_{\mathrm{ms}}
+
\lambda_{\mathrm{cyc}}\mathcal{L}_{\mathrm{cyc}}
+
\lambda_{\mathrm{cons}}\mathcal{L}_{\mathrm{cons}},
\label{eq_total_loss}
\end{equation}
where $\mathcal{L}_{\mathrm{ms}}$ teaches scale-specific structure to be reconstructable, $\mathcal{L}_{\mathrm{cyc}}$ restricts the time-frequency outputs to be mutually reconstructable through $W_\theta$ and $W_\theta^{-1}$, and $\mathcal{L}_{\mathrm{cons}}$ locks the two pathways to a shared full-rate prediction, so discrepancies at inference are less likely to arise only from pathway drift and can provide useful evidence for anomaly scoring.

\subsection{Tri-consistency Anomaly Scoring}
\label{inference}
At inference, LEFT scores anomalies as violations of agreement in two complementary forms: inconsistency within the time-frequency cycle, and reconstruction residuals across resolutions.

\begin{lemma}
\label{lem_score_closed_loop}
Assume Lemma~\ref{lem_tf_cycle} holds.
Assume further that the moving-average smoother $\mathrm{MA}_\kappa$ is a linear operator with
nonnegative kernel weights $\{w_i\}$ satisfying $\sum_i w_i=1$.
Moreover, assume there exists a constant $\rho_\kappa\ge 1$ (depending only on $\kappa$ and the boundary rule)
such that for any nonnegative sequence $u(t)$,
\begin{equation}
\Bigl\langle \bigl(\mathrm{MA}_\kappa(u)\bigr)(t)\Bigr\rangle_t
\ \le\
\rho_\kappa\,\langle u(t)\rangle_t.
\label{eq_ma_rho}
\end{equation}

Define the \emph{raw} cycle magnitude inside the smoother as:
\begin{equation}
\tilde{\mathcal{A}}_{\mathrm{cyc}}(t)
=
\alpha_f\left|\hat{\bm{X}}_{\leftarrow f}(t)-\bm{X}(t)\right|
+\alpha_t\left|\hat{\bm{X}}(t)-\bm{X}(t)\right|
+\alpha_g\,g(t)
+\alpha_c\,c(t),
\label{eq_acyc_raw}
\end{equation}
where $c(t)\ge 0$ is the cross-path discrepancy used in Eq.~\eqref{eq_score_cyc}
(e.g., $c(t)=\left|\hat{\bm{X}}_{\mathrm{ms}}(t)-\hat{\bm{X}}_{\leftarrow f}(t)\right|$).
Recall that $\mathcal{A}_{\mathrm{cyc}}(t)=\bigl(\mathrm{MA}_\kappa(\tilde{\mathcal{A}}_{\mathrm{cyc}})\bigr)(t)$ and
$\mathcal{A}(t)=\alpha_{\mathrm{cyc}}\mathcal{A}_{\mathrm{cyc}}(t)+\alpha_{\mathrm{ms}}\mathcal{A}_{\mathrm{ms}}(t)$.
If training achieves
$\mathcal{L}_{\mathrm{ms}}\le \varepsilon_{\mathrm{ms}}$,
$\mathcal{L}_{\mathrm{cyc}}\le \varepsilon_{\mathrm{cyc}}$,
and $\mathcal{L}_{\mathrm{cons}}\le \varepsilon_{\mathrm{cons}}$,
then the mean inference score satisfies,
\begin{equation}
\langle \mathcal{A}(t)\rangle_t
\ \le\
\kappa_{\mathrm{ms}}\varepsilon_{\mathrm{ms}}
+
\rho_\kappa\,\kappa_{\mathrm{cyc}}\varepsilon_{\mathrm{cyc}}
+
\rho_\kappa\,\kappa_{\mathrm{cons}}\varepsilon_{\mathrm{cons}}
+
\rho_\kappa\,\kappa_g\,\alpha_g\,\langle g(t)\rangle_t
+
\kappa_0,
\label{eq_normal_upper}
\end{equation}
where constants depend only on score aggregation weights, supervised scales, and $\rho_\kappa$.
Conversely, let $\Omega\subseteq\{1,\ldots,T\}$ be a non-negligible subset.
If
$\langle |\hat{\bm{X}}(t)-\bm{X}(t)| \rangle_{t\in\Omega} \ge \delta_t$,
$\langle |\hat{\bm{X}}_{\leftarrow f}(t)-\bm{X}(t)| \rangle_{t\in\Omega}\ge \delta_f$,
$\langle c(t)\rangle_{t\in\Omega}\ge \delta_c$,
then the cycle component obeys,
\begin{equation}
\bigl\langle \tilde{\mathcal{A}}_{\mathrm{cyc}}(t)\bigr\rangle_{t\in\Omega}
\ \ge\
\alpha_t\delta_t+\alpha_f\delta_f+\alpha_c\delta_c.
\label{eq_anom_lower_cycle}
\end{equation}

If in addition the pointwise bounds
$|\hat{\bm{X}}(t)-\bm{X}(t)|\ge \delta_t$,
$|\hat{\bm{X}}_{\leftarrow f}(t)-\bm{X}(t)|\ge \delta_f$,
and $c(t)\ge \delta_c$
hold for all $t\in\Omega$, then for any $t\in\Omega$ whose full averaging window of $\mathrm{MA}_\kappa$
lies inside $\Omega$, one has
$\mathcal{A}_{\mathrm{cyc}}(t)\ge \alpha_t\delta_t+\alpha_f\delta_f+\alpha_c\delta_c$.
\end{lemma}

From Lemma \ref{lem_score_closed_loop} (proof is given in Appendix \ref{sec_proof_lemma3}), we propose a timestamp-level anomaly map:
\begin{equation}
\mathcal{A}(t)=\alpha_{\mathrm{cyc}}\mathcal{A}_{\mathrm{cyc}}(t)+\alpha_{\mathrm{ms}}\mathcal{A}_{\mathrm{ms}}(t),
\label{eq_score_total}
\end{equation}
where $\mathcal{A}_{\mathrm{cyc}}(t)$ is the cycle-based discrepancy score and $\mathcal{A}_{\mathrm{ms}}(t)$ is the multi-scale reconstruction score.
We compute cycle residuals and apply a moving-average $\mathrm{MA}_\kappa$ with smoothing window size $\kappa$ to suppress local noise and stabilize timestamp-level decisions:
\begin{equation}
\mathcal{A}_{\mathrm{cyc}}(t)
=
\Big(\mathrm{MA}_\kappa\!\big(
\alpha_f\lvert\hat{\bm{X}}_{\leftarrow f}-\bm{X}\rvert
+\alpha_t\lvert\hat{\bm{X}}-\bm{X}\rvert
+\alpha_g\,g(t)
+\alpha_c\,c(t)
\big)\Big)(t),
\label{eq_score_cyc}
\end{equation}
where $c(t)$ measures cross-path discrepancy between the multi-scale pathway and the cycle pathway on the time domain, so disagreements penalized by $\mathcal{L}_{\mathrm{cons}}$ during training can be directly reflected in the inference score. $g(t)$ is computed from the sharpness and entropy of the time- and frequency-view prototype assignments.
Multi-scale score aggregates errors from all supervised scales by aligning them to the full resolution, so evidence that is salient at a coarse scale can still contribute to the fine-grained anomaly map:
\begin{equation}
\mathcal{A}_{\mathrm{ms}}(t)
=
\lvert\hat{\bm{X}}_{\mathrm{ms}}(t)-\bm{X}(t)\rvert
+
\sum_{k=2}^{K}\operatorname{Up}_{T}\!\left(\lvert\hat{\bm{X}}^{(k)}_{\mathrm{ms}}-\bm{X}^{(k)}\rvert\right)(t),
\label{eq_score_ms}
\end{equation}
where $\operatorname{Up}_{T}(\cdot)$ upsamples a sequence to length $T$.
This design avoids relying on a single-view dominant residual: any-view irregularity can increase $\mathcal{A}(t)$ either through cycle disagreement, cross-path mismatch, or cross-resolution residual accumulation.

\section{Experiments}
\subsection{Experimental Settings}
\textbf{Datasets.} 
\begin{table}[h]
  \centering
    \caption{Dataset statistics. AR denotes anomaly ratio.}
    \resizebox{1\columnwidth}{!}{
    \begin{tabular}{cccccccc}
    \hline
    \textbf{Dataset} & \textbf{Domain} & \textbf{Dimension} & \textbf{Window} & \textbf{Training} & \textbf{Validation} & \textbf{Test (labeled)} & \textbf{AR (\%)} \\
    \hline
    MSL & Spacecraft & 1 & 96 & 46,653 & 11,664 & 73,729 & 10.5 \\
    PSM & Server Machine & 25 & 192 & 105,984 & 26,497 & 87,841 & 27.8 \\
    SMAP & Spacecraft & 1 & 192 & 108,146 & 27,037 & 427,617 & 12.8 \\
    SMD & Server Machine & 38 & 192 & 566,724 & 141,681 & 708,420 & 4.2 \\
    SWaT & Water treatment & 31 & 192 & 396,000 & 99,000 & 449,919 & 12.1 \\
    GECCO & Water treatment & 9 & 128 & 55,408 & 13,852 & 69,261 & 1.25 \\
    SWAN & Space Weather & 38 & 192 & 48,000 & 12,000 & 60,000 & 23.8 \\
    \hline
  \end{tabular}
  }
\label{tab_datasets_details}
\end{table}
We evaluate LEFT on public TSAD benchmarks covering server monitoring, space telemetry, and industrial control. \textit{SMD} and \textit{PSM} contain server resource-utilization or performance metrics~\cite{su_robust_2019,abdulaal_practical_2021}; \textit{MSL} and \textit{SMAP} are NASA telemetry datasets with multivariate sensor and actuator signals~\cite{hundman_detecting_2018}; and \textit{SWaT} records water-treatment sensor traces under normal operation and attacks~\cite{mathur_swat_2016}. We also use \textit{NeurIPS-TS} and report results on its \textit{GECCO} and \textit{SWAN} subsets~\cite{wang_revisiting_2024}. For \textit{MSL} and \textit{SMAP}, we follow prior work~\cite{shentu_towards_nodate,li_crossad_nodate} and keep only the first continuous channel, since discrete variables are less suitable for reconstruction-based scoring. Dataset details are given in Table~\ref{tab_datasets_details}, and the data split and sliding-window setup follow~\cite{li_crossad_nodate}.

\textbf{Baselines.} Baselines include linear transformation-based methods, such as \textit{OCSVM}\cite{scholkopf_support_1999} and \textit{PCA}\cite{mei_ling_novel_2003}, and outlier-oriented detectors, including \textit{IForest}\cite{liu_isolation_2008} and \textit{LODA}\cite{pevny_loda_2016}.
We consider density-based methods (\textit{HBOS}\cite{goldstein_histogram_based_2012}, \textit{LOF}\cite{breunig_lof_2000}), which remain competitive in low-dimensional regimes but can be sensitive to complex temporal drift.
Neural baselines span reconstruction and forecasting, including \textit{AutoEncoder} (AE)\cite{sakurada_anomaly_2014}, \textit{DAGMM}\cite{zong_deep_2018}, \textit{LSTM}\cite{hundman_detecting_2018}, \textit{CAE-Ensemble} (CAE)\cite{campos_unsupervised_2021}, and \textit{Omni-Anomaly} (Omni)\cite{su_robust_2019}.
We include recent TSAD models from transformer, contrastive, frequency-aware, and pre-trained lines, such as \textit{Anomaly Transformer} (AT)\cite{xu_anomaly_nodate}, \textit{DCdetector} (DC)\cite{yang_dcdetector_2023}, \textit{ModernTCN}\cite{luo_moderntcn_2024}, \textit{GPT4TS}\cite{zhou_one_2023}, \textit{MtsCID}\cite{xie_multivariate_2025}, \textit{TimeMixer}\cite{wang_timemixer_nodate}, \textit{TimesNet}\cite{wu_timesnet_nodate}, \textit{CrossAD}\cite{li_crossad_nodate}, \textit{DADA}\cite{shentu_towards_2025}, and \textit{CATCH}\cite{wu_catch_2025}. Among them, \textit{CATCH} is a recent frequency-aware baseline, \textit{DADA} is a recent pre-trained general baseline, and \textit{MtsCID}, \textit{TimeMixer}, \textit{TimesNet}, and \textit{CrossAD} are closely related multi-scale SOTA methods.

\textbf{Metrics.} Evaluation in TSAD can be misleading under point adjustment, since even random predictors may obtain inflated scores~\cite{yang_dcdetector_2023, xu_anomaly_nodate}.
TSB-AD~\cite{liu_elephant_2024} further argues that \textit{VUS-PR}~\cite{paparrizos_volume_2022} provides a more robust and less lag-sensitive evaluation, while several common metrics can be biased across scenarios.
We therefore report \textit{VUS-PR} and \textit{VUS-ROC} as the main metrics
~\cite{paparrizos_volume_2022}.

\textbf{Implementation Details.}
In our setup, the time-view encoder is a lightweight 1D convolutional stack with kernel sizes $(5,3,1)$, while the frequency-view encoder is a lightweight 2D convolutional stack with two $(3,3)$ layers. The multi-scale branch uses a fixed set of downsampling factors $\{r_k\}_{k=1}^K$ and the number of scales $K$ is given by the number of downsampling factors. All decoders are two-layer MLPs. We use sliding windows for both training and inference, and report results under a non-overlapping window protocol. 
Optimization uses Adam with learning rate $10^{-4}$ and batch size 128. Following~\cite{qiu_tfb_2024}, we keep all test windows at inference time and do not apply the drop-last trick. We set decision thresholds with SPOT~\cite{siffer_anomaly_2017}, consistent with prior TSAD practice~\cite{shentu_towards_nodate, li_crossad_nodate}. All experiments are implemented in PyTorch and run on a single NVIDIA L40 GPU. More details and code are available at \href{https://github.com/DezhengWang/Left.git}{https://github.com/DezhengWang/Left.git}.
\subsection{Detection Results}
\begin{table*}[thbp]
\centering
\caption{Results in the seven real-world datasets. V-R and V-P denote VUS-ROC and VUS-PR, where higher values indicate better performance. The best ones are in bold, and the second ones are underlined. $\ddag$ indicates statistical significance by t-test with $p<0.05$. $\dag$ indicates statistical significance by t-test with $p<0.1$. }
\resizebox{\linewidth}{!}{
\begin{tabular}{c|cc|cc|cc|cc|cc|cc|cc}
\hline
Dataset & \multicolumn{2}{c|}{SMD} & \multicolumn{2}{c|}{MSL} & \multicolumn{2}{c|}{SMAP} & \multicolumn{2}{c|}{SWaT} & \multicolumn{2}{c|}{PSM} & \multicolumn{2}{c|}{GECCO} & \multicolumn{2}{c}{SWAN} \\
Metric & V-R & V-P & V-R & V-P & V-R & V-P & V-R & V-P & V-R & V-P & V-R & V-P & V-R & V-P \\
\hline
OCSVM & 
0.6451 & 0.1131 & 
0.5798 & 0.1753 & 
0.4185 & 0.1133 & 
0.5903 & 0.4396 & 
0.5993 & 0.4252 & 
0.7533 & 0.1207 & 
0.9088 & 0.9004 \\
PCA & 
0.7174 & 0.1529 & 
0.6108 & 0.1889 & 
0.4090 & 0.1144 & 
0.6149 & 0.4459 & 
0.6331 & 0.4706 & 
0.5366 & 0.0443 & 
0.9290 & 0.9123 \\
IForest & 
0.7224 & 0.1304 & 
0.5638 & 0.1631 & 
0.4960 & 0.1315 & 
0.3677 & 0.1011 & 
0.6009 & 0.3964 & 
0.7083 & 0.0943 & 
0.8835 & 0.8793 \\
LODA & 
0.6745 & 0.1213 & 
0.5375 & 0.1689 & 
0.3973 & 0.1017 & 
0.6358 & 0.3531 & 
0.6089 & 0.4423 & 
0.5749 & 0.0339 & 
0.9170 & 0.9107 \\
HBOS & 
0.6670 & 0.1102 & 
0.6265 & 0.1790 & 
0.5620 & 0.1388 & 
0.7084 & 0.4602 & 
0.7056 & 0.5061 & 
0.5440 & 0.0453 & 
0.9056 & 0.8894 \\
LOF & 
0.6893 & 0.1076 & 
0.6081 & 0.1715 & 
0.5673 & 0.1409 & 
0.6667 & 0.4187 & 
0.6628 & 0.4615 & 
0.7817 & 0.0919 & 
0.9095 & 0.9007 \\
\hline
AE & 
0.7560 & 0.1542 & 
0.6047 & 0.1890 & 
0.4687 & 0.1366 & 
0.5903 & 0.4144 & 
0.6339 & 0.4490 & 
0.6124 & 0.0448 & 
0.6982 & 0.0201 \\
DAGMM & 
0.6988 & 0.1496 & 
0.6069 & 0.1803 & 
0.5599 & 0.1349 & 
0.5746 & 0.4731 & 
0.5598 & 0.4522 & 
0.5099 & 0.0396 & 
0.8951 & 0.8697 \\
LSTM & 
0.7001 & 0.1395 & 
0.6163 & 0.1681 & 
0.5329 & 0.1399 & 
0.5482 & 0.2200 & 
0.5571 & 0.4592 & 
0.6450 & 0.0668 & 
0.9082 & 0.8862 \\
CAE & 
0.7174 & 0.1376 & 
0.5382 & 0.1639 & 
0.4212 & 0.1140 & 
0.5939 & 0.4104 & 
0.6113 & 0.4395 & 
0.5524 & 0.0528 & 
0.9042 & 0.9022 \\
Omni & 
0.7080 & 0.1340 & 
0.5490 & 0.1973 & 
0.4743 & 0.1239 & 
0.6187 & 0.4475 & 
0.6340 & 0.4472 & 
0.5386 & 0.0517 & 
0.9041 & 0.9022 \\
AT & 
0.5117 & 0.0796 & 
0.3890 & 0.1041 & 
0.4571 & 0.1239 & 
0.5561 & 0.2679 & 
0.5186 & 0.3309 & 
0.4751 & 0.0278 & 
0.8046 & 0.7943 \\
DC & 
0.5145 & 0.0814 & 
0.3900 & 0.0948 & 
0.4444 & 0.1149 & 
0.5191 & 0.1495 & 
0.5235 & 0.3366 & 
0.5454 & 0.0361 & 
0.8429 & 0.8338 \\
GPT4TS & 
0.7679 & 0.1745 & 
0.7697 & 0.2769 & 
0.5449 & 0.1289 & 
0.2537 & 0.0846 & 
0.6466 & 0.4599 & 
0.9776 & 0.4181 & 
0.9340 & 0.8924 \\
ModernTCN & 
0.7707 & 0.1596 & 
0.7747 & 0.3010 & 
0.5470 & 0.1395 & 
0.2735 & 0.0941 & 
0.6480 & 0.4668 & 
0.9694 & 0.4819 & 
0.9027 & 0.8962 \\
MtsCID & 
0.5162 & 0.0815 & 
0.4686 & 0.1181 & 
0.4260 & 0.1177 & 
0.5021 & 0.1283 & 
0.5194 & 0.3296 & 
0.5315 & 0.0381 & 
0.8128 & 0.8375 \\
TimeMixer & 
0.7711 & 0.1391 & 
0.7858 & 0.2461 & 
0.5552 & 0.1371 & 
0.2673 & 0.0918 & 
0.5974 & 0.3807 & 
0.9899 & 0.4606 & 
0.9290 & 0.8721 \\
TimesNet & 
0.8420 & 0.2040 & 
0.7880 & 0.2731 & 
0.5495 & 0.1352 & 
0.2974 & 0.1158 & 
0.6344 & 0.4373 & 
0.9834 & 0.4578 & 
0.9515 & 0.9160 \\
DADA & 
0.7249 & 0.1188 & 
0.5457 & 0.1758 & 
0.4307 & 0.1151 & 
0.6216 & 0.4349 & 
0.6583 & 0.4707 & 
0.6835 & 0.0605 & 
\underline{0.9529} & 0.9081\\
CATCH & 
0.8397 & 0.1951 & 
0.8042 & 0.2971 & 
0.5648 & \underline{0.1479} & 
0.2599 & 0.0982 & 
0.6886 & 0.4888 & 
0.9890 & 0.4512 & 
0.9467 & \underline{0.9189}\\
CrossAD & 
\underline{0.8580} & \underline{0.2344} & 
\underline{0.8091} & \underline{0.3144} & 
\underline{0.5779} & 0.1443 & 
\underline{0.7865} & \underline{0.4767} & 
\underline{0.7302} & \underline{0.5596} & 
\underline{0.9948} & \underline{0.6211} & 
0.9499 & 0.9171 \\
\hline
\textbf{LEFT} & 
\textbf{0.8638} & \textbf{0.2389} & 
\textbf{0.8157}$^{\ddag}$ & \textbf{0.3342}$^{\ddag}$ & 
\textbf{0.6562}$^{\ddag}$ & \textbf{0.1726}$^{\ddag}$ & 
\textbf{0.7991}$^{\ddag}$ & \textbf{0.5290}$^{\ddag}$ & 
\textbf{0.7836}$^{\ddag}$ & \textbf{0.5701}$^{\dag}$ & 
\textbf{0.9949}$^{\ddag}$ & \textbf{0.6634}$^{\ddag}$ & 
\textbf{0.9549}$^{\ddag}$ & \textbf{0.9244}$^{\ddag}$ \\
\hline
\end{tabular}
}
\label{tab_main_results}
\end{table*}

We evaluate LEFT on seven real-world TSAD benchmarks and compare it with 21 competitive baselines, as reported in Table~\ref{tab_main_results}. Across datasets with varied dynamics and anomaly types, LEFT achieves stronger ranking-based detection quality under both VUS-PR and VUS-ROC. As shown in Table~\ref{tab_main_results}, LEFT improves both metrics on every benchmark over the strongest baseline (underlined), with average gains of about 0.023 in both VUS-ROC and VUS-PR. Relative to the strongest baseline on each dataset, LEFT improves VUS-ROC by about 3.45\% and VUS-PR by about 6.45\% on average. The largest relative gains appear on SMAP, where LEFT improves VUS-ROC by 13.55\% and VUS-PR by 16.70\% relative to the strongest baseline. 

For a comprehensive comparison, we report additional evaluation metrics in Table~\ref{tab_main_results_details}. We include AUC-ROC (AUC-R), AUC-PR (AUC-P), Range-AUC-ROC (R-A-R), Range-AUC-PR (R-A-P), Accuracy (Acc), and Standard F1 (F1). Here we compare only with CrossAD~\cite{li_crossad_nodate}, since it is the best-performing baseline in our main results (see Table~\ref{tab_main_results}). The results show that LEFT performs strongly across these metrics, which further supports its effectiveness for time series anomaly detection.
\begin{table}[thbp]
\centering
\caption{Multi-metrics results in the three real-world datasets.}
\resizebox{.9\linewidth}{!}{
\begin{tabular}{cccccccc}
\hline
Dataset & Model & AUC-R & AUC-P & R-A-R & R-A-P & ACC & F1\\\hline
\multirow{2}{*}{PSM} & CrossAD & 0.6523 & 0.4706 & 0.7292 & 0.5733 & 0.6274 & 0.4675\\
& LEFT & \textbf{0.7755} & \textbf{0.5753} & \textbf{0.7858} & \textbf{0.5890} & \textbf{0.6555} & \textbf{0.5693}\\
\multirow{2}{*}{MSL} & CrossAD & 0.7808 & 0.2837 & 0.8095 & 0.3409 & 0.6786 & 0.3361 \\
& LEFT & \textbf{0.7878} & \textbf{0.3019} & \textbf{0.8121} & \textbf{0.3546} & \textbf{0.7296} & \textbf{0.3436} \\
\multirow{2}{*}{SMAP} & CrossAD & 0.5577 & 0.1333 & 0.5926 & 0.1527 & 0.4285 & 0.2792 \\
& LEFT & \textbf{0.6411} & \textbf{0.1618} & \textbf{0.6667} & \textbf{0.1821} & \textbf{0.5842} & \textbf{0.3053} \\
\hline
\end{tabular}
}
\label{tab_main_results_details}
\end{table}

\subsection{Ablation Studies}
\label{sec_ablation}
\begin{table}[thbp]
\centering
\caption{Ablation studies for LEFT. Higher values indicate better performance. `$\downarrow$' marks a severe (over 5\%) performance drop compared with LEFT.}
\resizebox{\linewidth}{!}{
\begin{tabular}{c|cc|cc|cc}
\hline
\multirow{2}{*}{Architecture} & \multicolumn{2}{c|}{PSM} & \multicolumn{2}{c|}{MSL} & \multicolumn{2}{c}{SMAP} \\
& V-R & V-P & V-R & V-P & V-R & V-P 
\\\hline
$w/o$ Tri-view Interaction & 0.7814 & 0.5653 & 0.8142 & 0.3235 & 0.5814$\downarrow$ & 0.1477$\downarrow$ \\
$w/o$ Learnable Filterbank & 0.7748 & 0.5564 & 0.8082 & 0.3254 & 0.5952$\downarrow$ & 0.1549$\downarrow$ \\
$w/o$ Cycle Consistency & 0.6497$\downarrow$ & 0.4407$\downarrow$ & 0.5283$\downarrow$ & 0.1666$\downarrow$ & 0.4487$\downarrow$ & 0.1161$\downarrow$ \\
$w/o$ Cross-path Consistency & 0.7753 & 0.5516 & 0.8100 & 0.3251 & 0.5946$\downarrow$ & 0.1480$\downarrow$\\
LEFT & \textbf{0.7836} & \textbf{0.5701} & \textbf{0.8157} & \textbf{0.3342} & \textbf{0.6562} & \textbf{0.1726}\\
\hline
\end{tabular}
}
\label{tab_ablation_results}
\end{table}

As shown in Table \ref{tab_ablation_results}, to evaluate the components of LEFT, we conduct detailed ablations. The results show that LEFT achieves the best performance, which indicates that the gains come from the combined effect of the proposed modules rather than any single component. Removing tri-view interaction reduces the average from 0.7518 / 0.3590 to 0.7256 / 0.3455, showing that cross-view alignment contributes to detection quality. Replacing the learnable filterbank with a fixed one causes a similar drop to 0.7261 / 0.3456 on average, suggesting that adaptive multi-scale structure captures evidence that fixed band splits miss in this setting. The largest effect comes from cycle consistency, since removing it collapses average performance to 0.5422 / 0.2411, which supports the view that many anomalies are better exposed by agreement breaks between the time and frequency paths than by a dominant single-branch residual. Removing cross-path consistency lowers the average to 0.7266 / 0.3416, which is consistent with its role in aligning the multi-scale pathway with the time–frequency pathway. In Appendix \ref{sec_ablation_appendix}, we provide a more comprehensive ablation analysis.

\subsection{Model Efficiency} 
\label{sec_efficiency}

We compare the efficiency of LEFT with representative TSAD methods on PSM. The baselines cover several common design lines in recent work. AnomalyTransformer~\cite{xu_anomaly_nodate} uses a transformer and detects anomalies through association discrepancy. CrossAD~\cite{li_crossad_nodate} uses a transformer for multi-scale modeling. TimeMixer~\cite{wang_timemixer_nodate} uses MLP-based multi-scale mixing. TimesNet~\cite{wu_timesnet_nodate} uses frequency-aware convolutional modeling. 
Table~\ref{tab_efficiency_results} shows that LEFT achieves the lowest FLOPs and the fastest measured runtime among the compared methods, while maintaining a moderate model size. LEFT uses 188.849M FLOPs and 4.256M parameters. LEFT runs in 25.66 seconds per training epoch and 0.01527 seconds per 128 samples at inference. LEFT also shows clear efficiency advantages over strong baselines.
The results suggest that LEFT stays small enough for typical deployment budgets, while its compute profile leads to clear wall-clock gains during both training and inference.

\subsection{Model Analysis}

\subsubsection{Fusion Strategy Analysis} 
\label{sec_fusion_strategy}

We compare all-pairs fusion and two-view variants. All-pairs fusion, denoted as $M\leftrightarrow F\leftrightarrow T$, updates each view by aggregating information from the other two views, followed by residual addition and layer normalization. The two-view variants $M\leftrightarrow F$, $M\leftrightarrow T$, and $T\leftrightarrow F$ enable interaction only within the selected pair of views, while keeping the third view unchanged. $T\leftrightarrow F$ exchanges information only between time and frequency tokens, $M\leftrightarrow T$ exchanges information only between multi-scale and time tokens, and $M\leftrightarrow F$ exchanges information only between multi-scale and frequency tokens. The impact of fusion strategies on detection performance is reported in Table~\ref{tab_fusion_results}. Across PSM, MSL, and SMAP, the proposed tri-view fusion yields competitive overall performance. It consistently improves over the best two-view variant in both VUS-ROC and VUS-PR, showing that using evidence from all three views jointly is more effective than relying on any pair of views. All-pairs fusion is not always better. It adds many cross-view links, which can mix the views too strongly and spread noise across branches. This can weaken discrepancy signals and hurt performance on some datasets. In contrast, the proposed fusion uses a selected set of interactions, which is enough to align the views, but it keeps view-specific information separated so discrepancies stay meaningful for anomaly scoring. In Appendix~\ref{sec_fusion_results_appendix}, we provide a more comprehensive analysis.

\begin{table}[t]
\centering
\caption{Comparison of various methods w.r.t FLOPs, \# params, training time, inference time. \# params represents the total number of trainable parameters. Training time is measured per epoch with batch size 128, while inference time is measured per batch of 128 samples.}
\resizebox{.95\linewidth}{!}{
\begin{tabular}{crrrr}
\hline
\multirow{2}{*}{Method} & \multirow{2}{*}{FLOPs} & \multirow{2}{*}{\# Params} & \multicolumn{1}{c}{Training} & \multicolumn{1}{c}{Inference}\\
& & & \multicolumn{1}{c}{Time (s)} & \multicolumn{1}{c}{Time (s)}\\
\hline
TimesNet & 21.145 G & 36.734 M & 285.80 & 0.10372\\
TimeMixer & 11.967 G & 1.349 M & 630.19 & 0.28682 \\
AnomalyTransformer & 920.912 M & 4.799 M & 59.86 & 0.03076\\
CrossAD & 1.288 G & 927.750 K & 205.22 & 0.06712 \\
LEFT & 188.849 M & 4.256 M & 25.66 & 0.01527 \\
\hline
\end{tabular}
}
\label{tab_efficiency_results}
\end{table}

\begin{table}[t]
\centering
\caption{Fusion strategy ablation.}
\resizebox{.95\linewidth}{!}{
\begin{tabular}{ccccccc}
\hline
\multirow{2}{*}{Fusion Method} & \multicolumn{2}{c}{PSM} & \multicolumn{2}{c}{MSL} & \multicolumn{2}{c}{SMAP} \\
& V-R & V-P & V-R & V-P & V-R & V-P\\
\hline
$M\leftrightarrow F\leftrightarrow T$ & \textbf{0.7860} & 0.5692 & 0.8106 & 0.3246 & 0.4776$\downarrow$ & 0.1295$\downarrow$ \\
$M\leftrightarrow F$ & 0.7687 & 0.5479 & 0.8069 & 0.3255 & 0.5640$\downarrow$ & 0.1419$\downarrow$ \\
$M\leftrightarrow T$ & 0.7558 & 0.5371$\downarrow$ & 0.8051 & 0.3149$\downarrow$ & 0.4330$\downarrow$ & 0.1163$\downarrow$ \\
$T\leftrightarrow F$ & 0.7790 & 0.5593 & 0.8121 & 0.3265 & 0.6203$\downarrow$ & 0.1634$\downarrow$ \\
Default & 0.7836 & \textbf{0.5701} & \textbf{0.8157} & \textbf{0.3342} & \textbf{0.6562} & \textbf{0.1726} \\
\hline
\end{tabular}
}
\label{tab_fusion_results}
\end{table}

\subsubsection{Fusion Hyper-parameters Analysis} 

\begin{figure}[thbp]
\centering
\includegraphics[width=0.31\columnwidth]{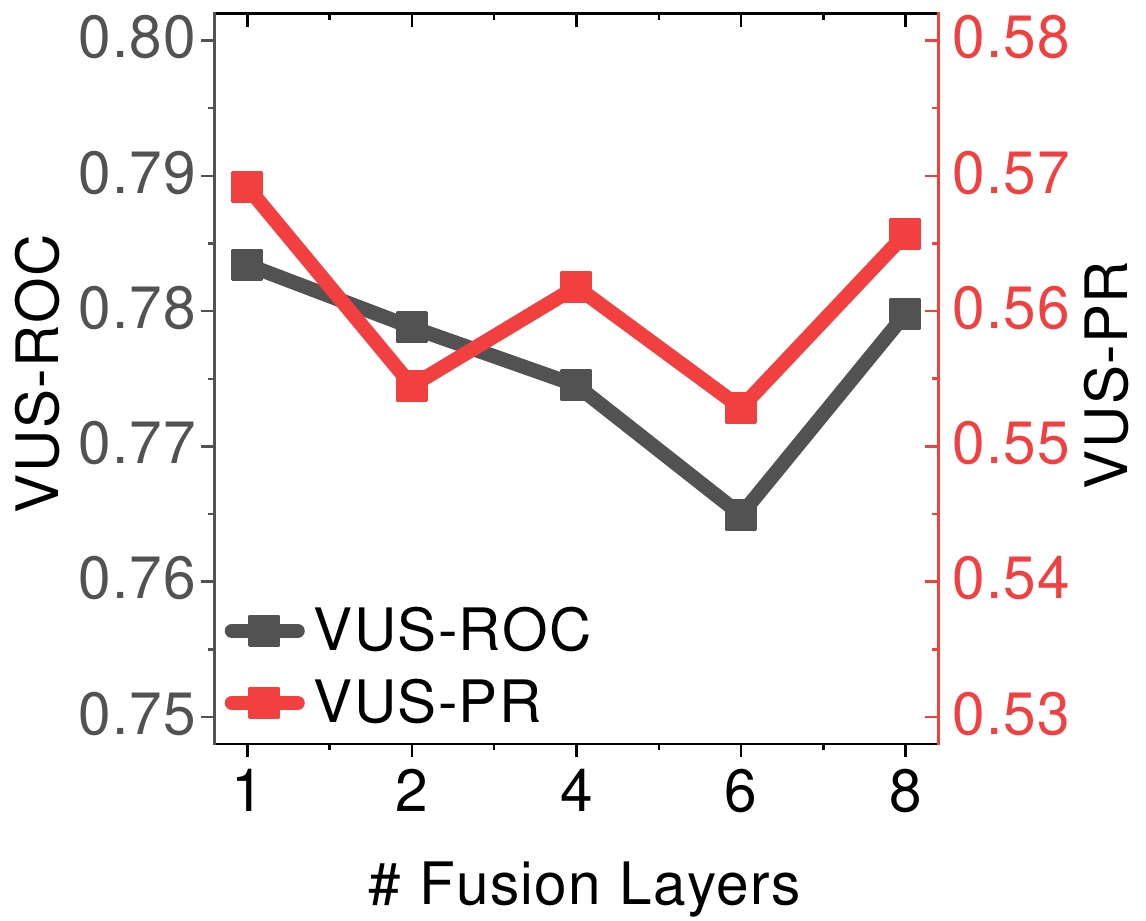}
\llap{(a)~}
\includegraphics[width=0.31\columnwidth]{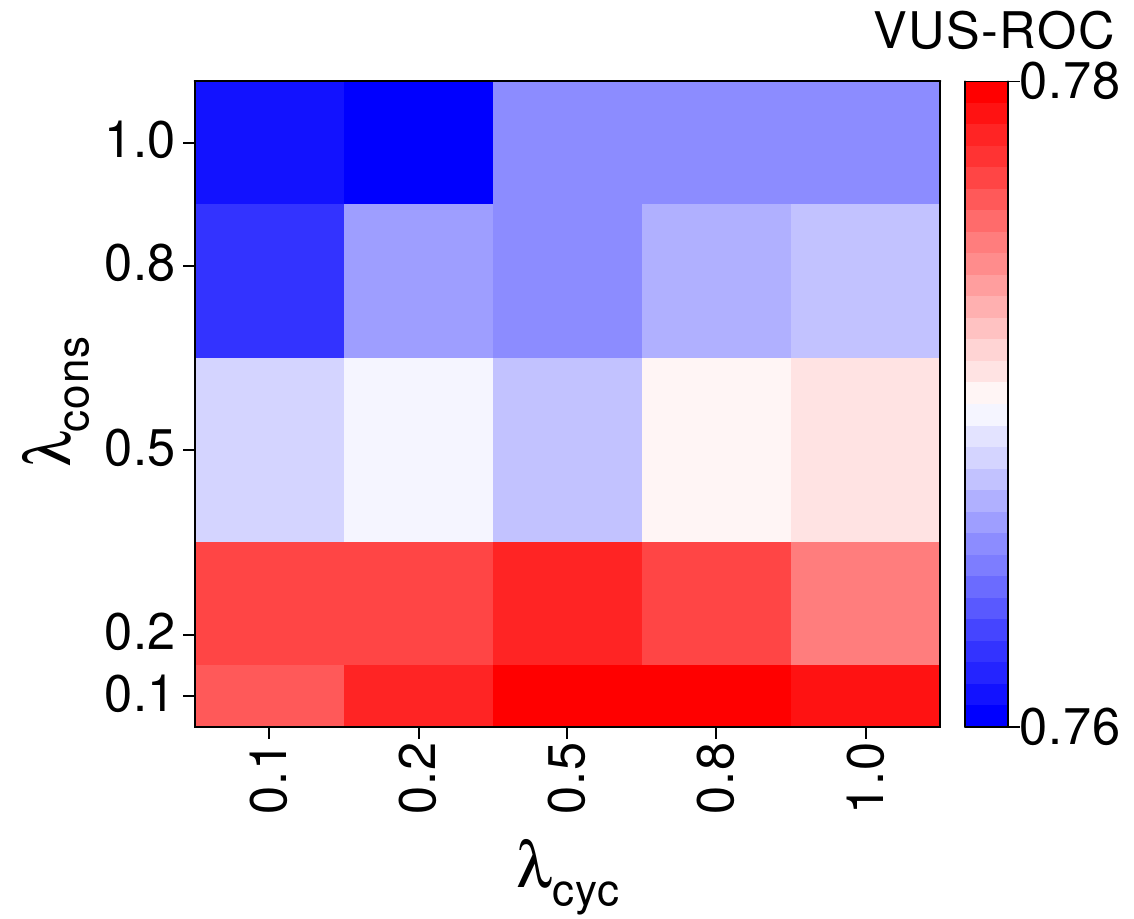}
\llap{(c)~}
\includegraphics[width=0.31\columnwidth]{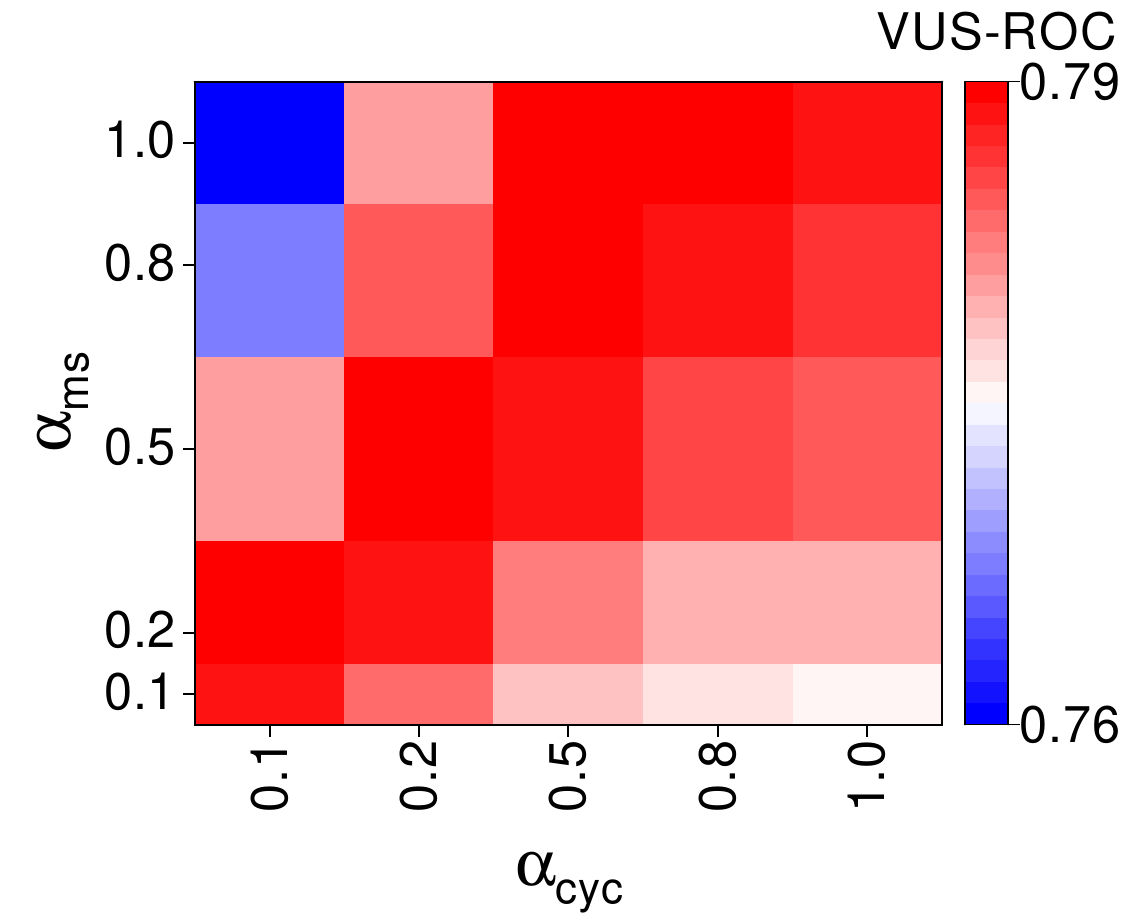}
\llap{(e)~}

\includegraphics[width=0.31\columnwidth]{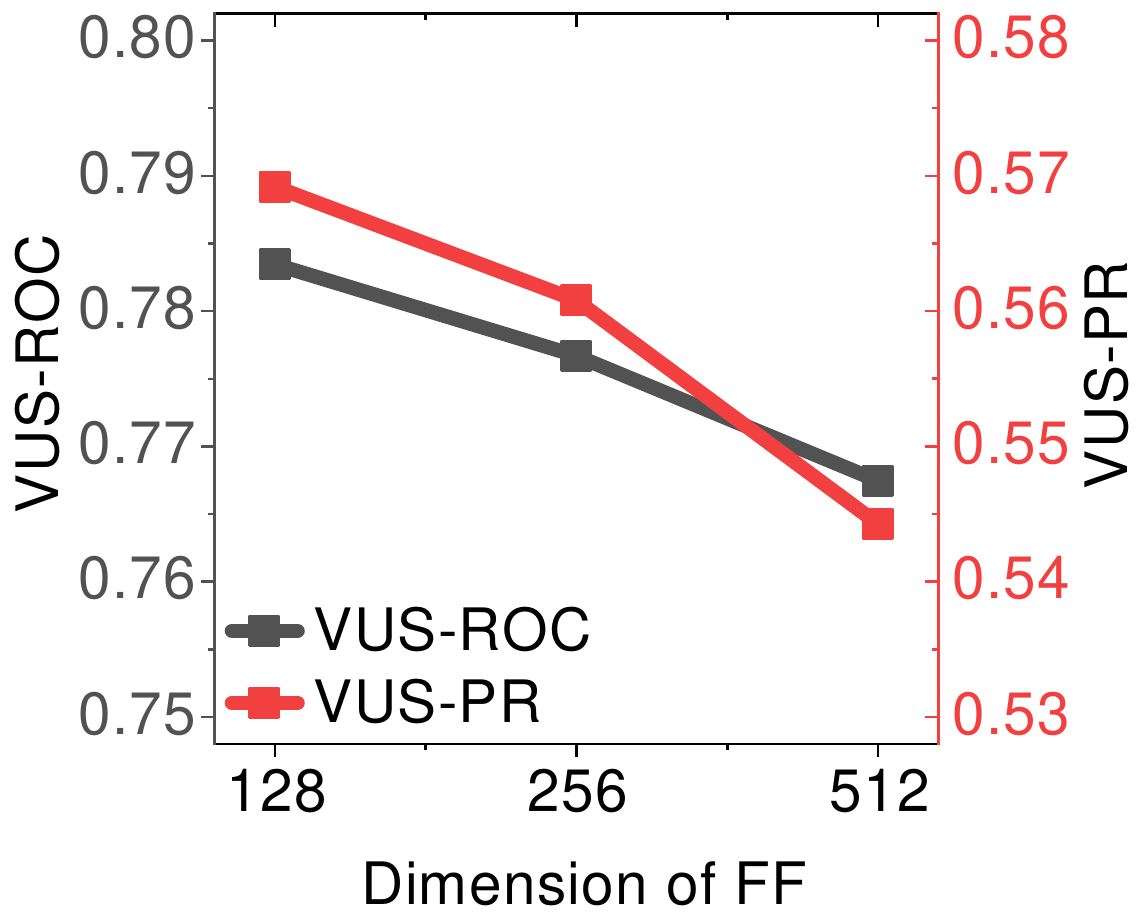}
\llap{(b)~}
\includegraphics[width=0.31\columnwidth]{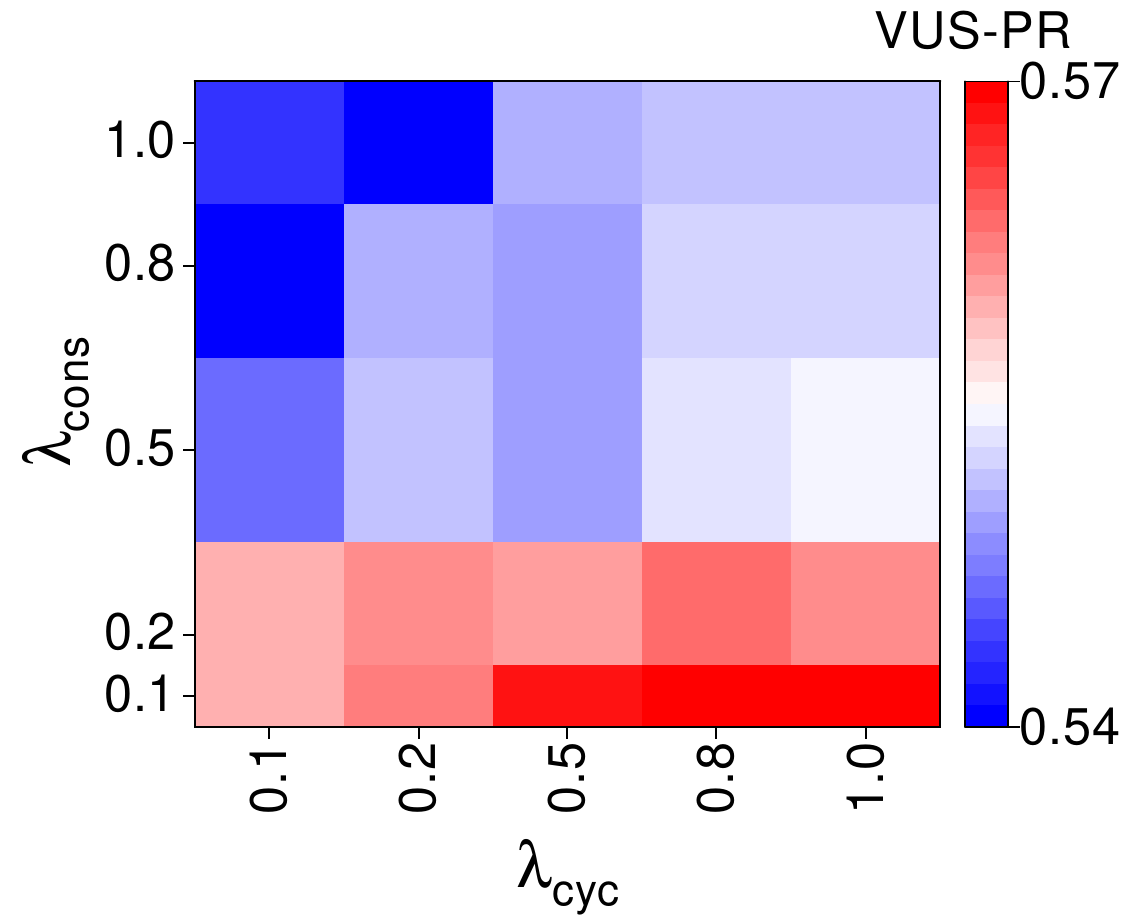}
\llap{(d)~}
\includegraphics[width=0.31\columnwidth]{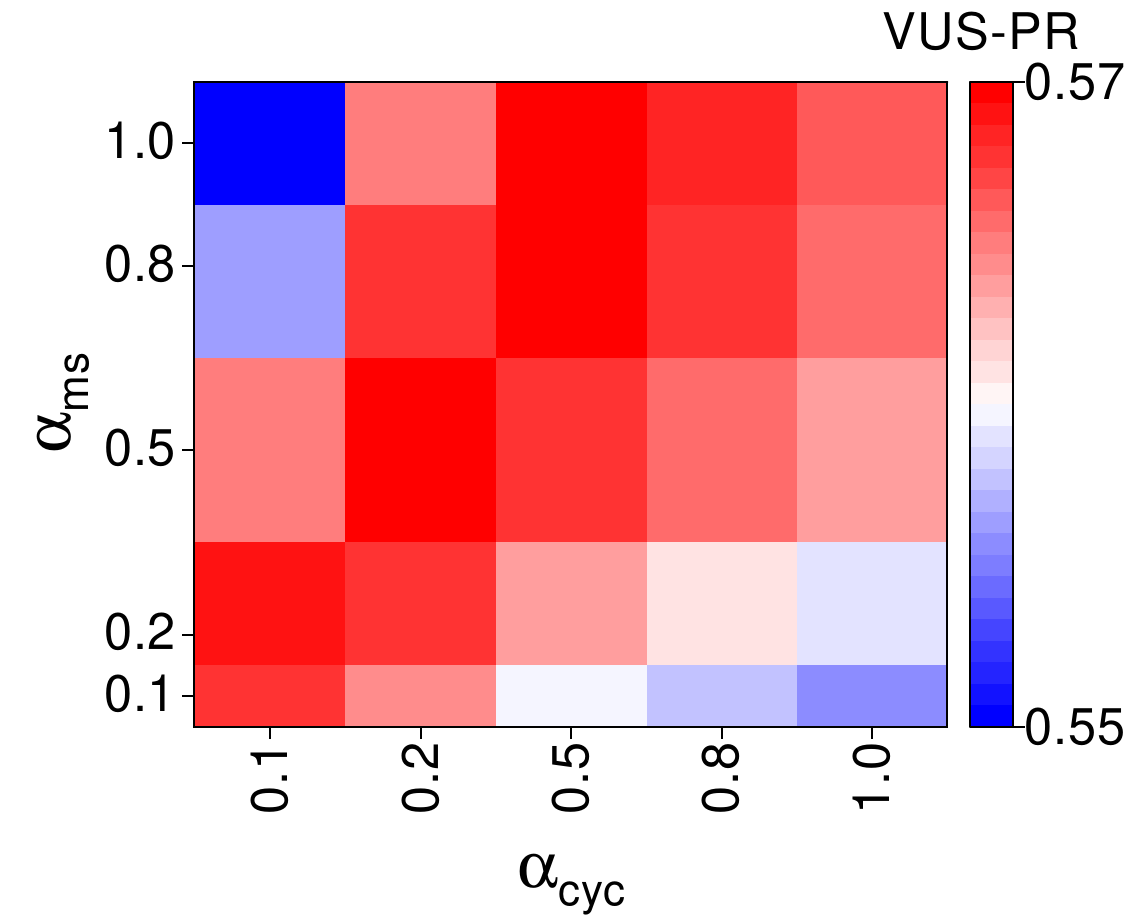}
\llap{(f)~}
\caption{Fusion hyper-parameters sensitivity analysis on PSM. (a, b) present the depth and width analysis. (c, d) show the sensitivity to $\lambda_{\mathrm{cyc}}$ and $\lambda_{\mathrm{cons}}$. (e, f): $\alpha_{\mathrm{cyc}}$ and  $\alpha_{\mathrm{ms}}$ sensitivity.  More results are listed in Appendix~\ref{sec_cyc_cons_appendix}.}
\label{fig_fusion_hyper_psm}
\end{figure}

\begin{figure}[thbp]
\centering
\includegraphics[width=.31\columnwidth]{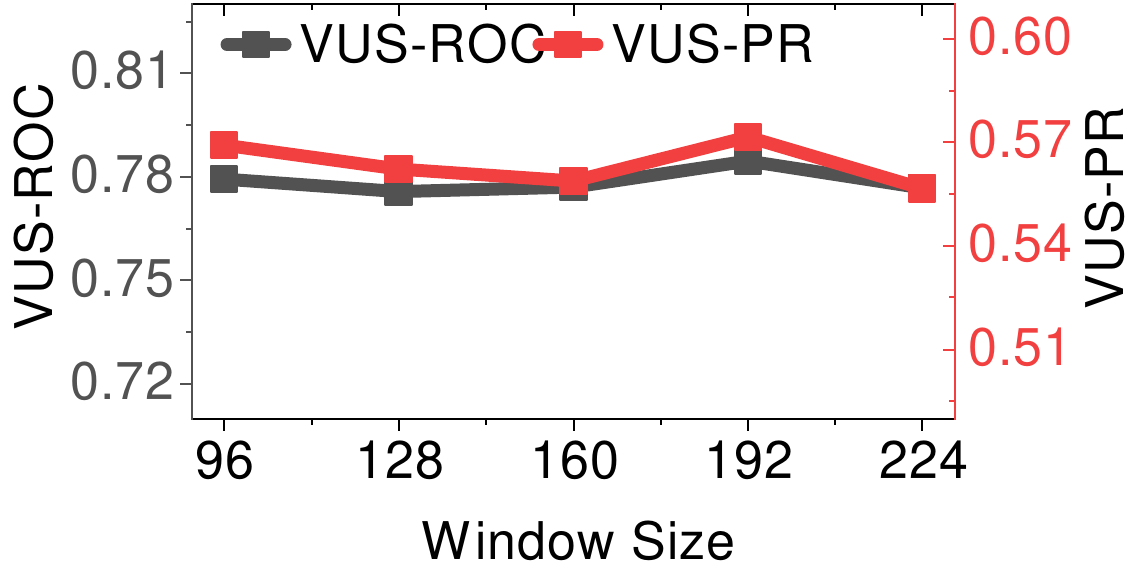}
\llap{(a)~}
\includegraphics[width=.31\columnwidth]{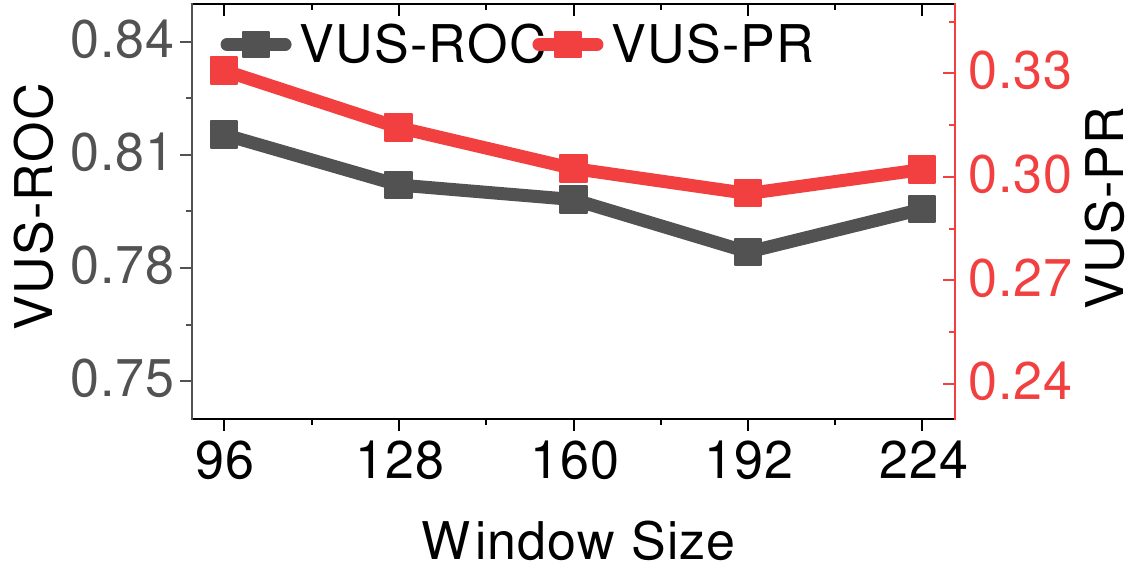}
\llap{(b)~}
\includegraphics[width=.31\columnwidth]{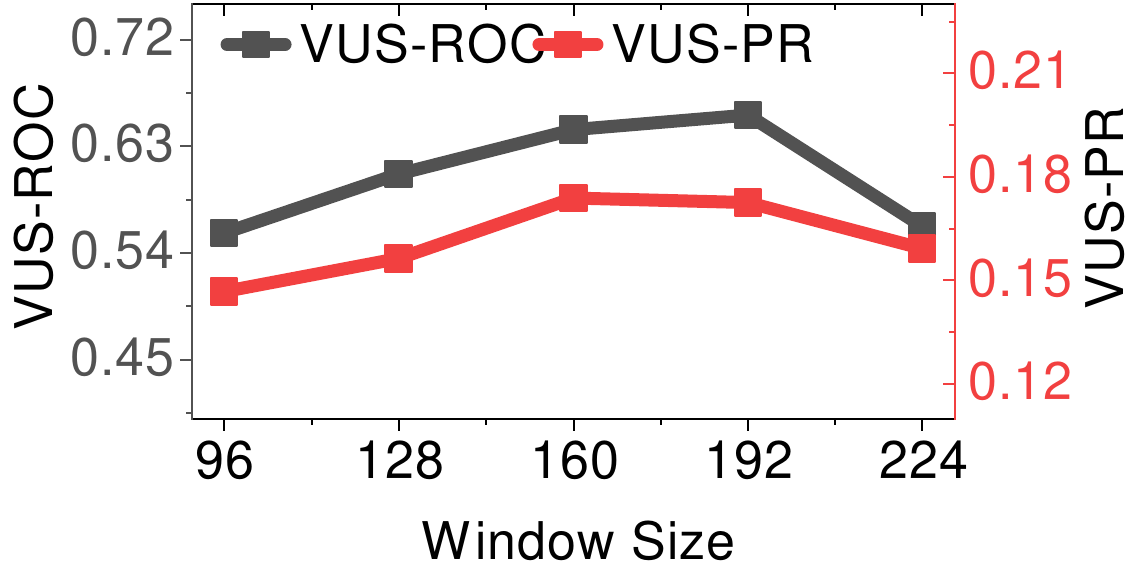}
\llap{(c)~}
\caption{Effect of window size. (a): PSM. (b): MSL. (c): SMAP. }
\label{fig_effect_seqlen}
\end{figure}

\begin{figure}[thbp]
\centering
\includegraphics[width=0.31\columnwidth]{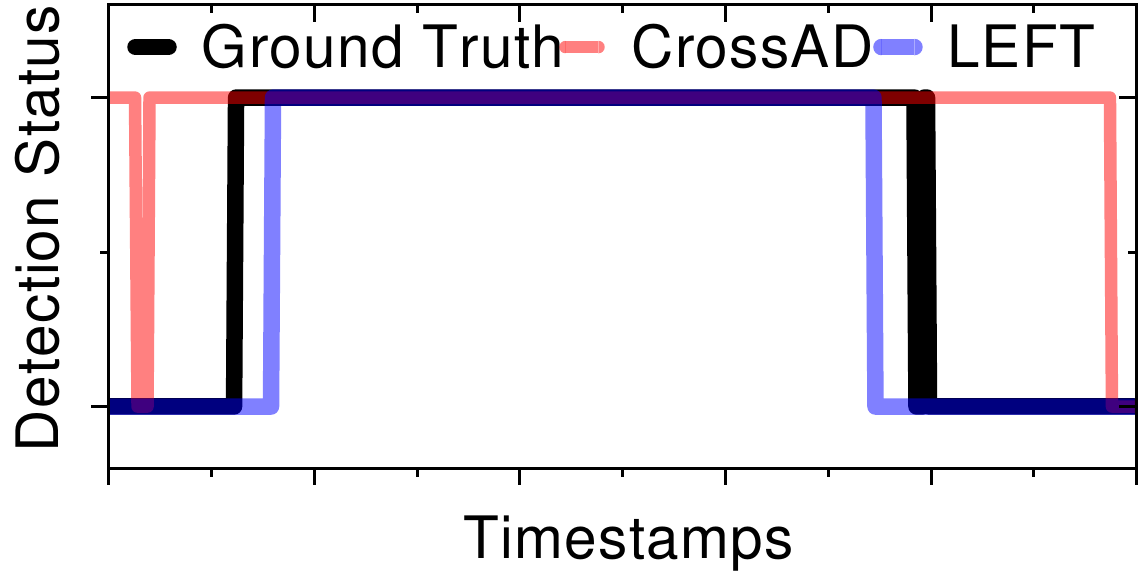}
\llap{(a)~}
\includegraphics[width=0.31\columnwidth]{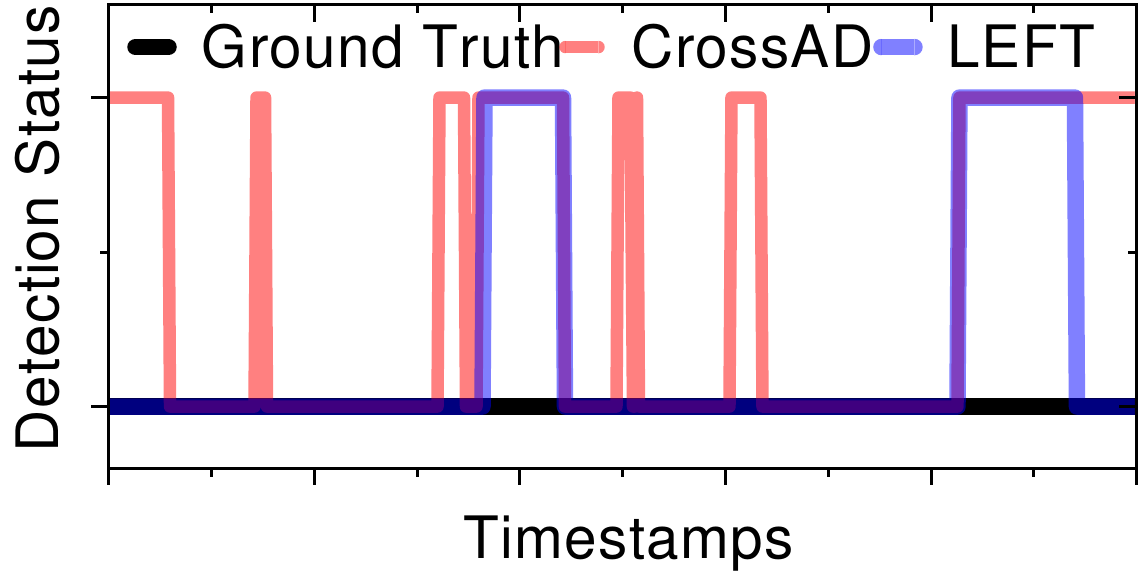}
\llap{(b)~}
\includegraphics[width=0.31\columnwidth]{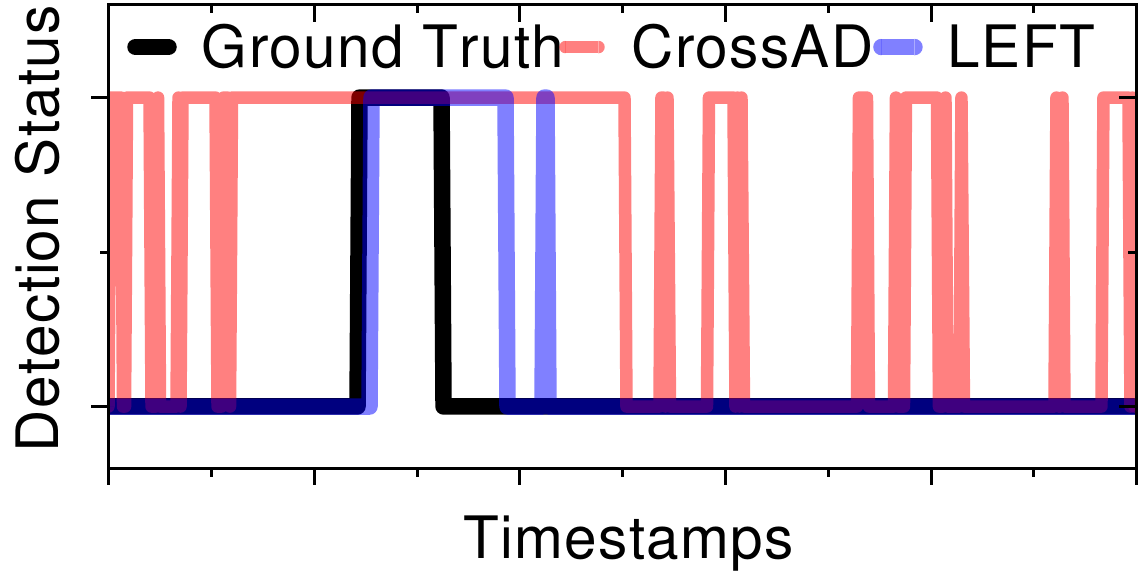}
\llap{(c)~}
\caption{Qualitative anomaly detection results. (a): PSM. (b): MSL. (c): SMAP. High denotes anomalies, and low denotes normal ones.}
\label{fig_example}
\end{figure}

Figs.~\ref{fig_fusion_hyper_psm} (a, b) show that increasing the depth or width of the fusion layer does not produce a meaningful performance gain, which aligns with the observation in Sec. \ref{sec_fusion_strategy} that stronger fusion is not always better. In this setting, a heavier fusion tends to over-mix the views, so distinctive evidence becomes less separable in the shared representation.

\subsubsection{$\lambda_{\mathrm{cyc}}$ and  $\lambda_{\mathrm{cons}}$ Sensitivity} 
\label{sec_cyc_cons}

Figs.~\ref{fig_fusion_hyper_psm} (c, d) suggest that detection quality depends on the choice of cycle consistency and cross-path consistency. LEFT remains competitive when cycle consistency is strong and cross-path consistency is relatively weak. Pushing cross-path consistency to large values does not yield stable gains. The results suggest that an overly strong cross-path consistency weight can over-align the pathways and weaken fault evidence. A stronger cycle consistency weight is often beneficial because it better exploits the analysis-synthesis link between time and frequency views. 

\subsubsection{$\alpha_{\mathrm{cyc}}$ and  $\alpha_{\mathrm{ms}}$ Sensitivity} 
\label{sec_alpha_score}
Figs.~\ref{fig_fusion_hyper_psm} (e, f) show that both VUS-ROC and VUS-PR depend on the choice of $\alpha_{\mathrm{cyc}}$ and $\alpha_{\mathrm{ms}}$, and neither metric improves monotonically when only one weight is increased. High-score regions concentrate in a middle range of the grid, while several extreme settings lead to lower values. This pattern echoes our main contribution that LEFT gains from combining cycle disagreement and reconstruction disagreement with multi-scale residual evidence, where reliable detection comes from their complementary use rather than amplifying a single term.

\subsubsection{Effect of Window Size} 
As shown in Fig.~\ref{fig_effect_seqlen}, LEFT shows stable performance under different window sizes, with dataset-dependent optima. PSM changes only slightly and peaks at 192, MSL favors shorter windows, while SMAP improves up to 192 but drops at 224. These results suggest that LEFT is robust within a practical window-size range, though moderate dataset-specific tuning remains useful.

\subsection{Visualization} 

Fig.~\ref{fig_example} presents qualitative results on PSM, MSL, and SMAP. In these examples, CrossAD assigns elevated scores to many normal timestamps, leading to frequent off-interval alarms and more false positives. LEFT produces more concentrated predictions and aligns more closely with the annotated anomalous intervals on PSM and SMAP, although false alarms still remain in some cases.

\section{Conclusion}
We propose LEFT, a unified tri-view framework for unsupervised TSAD. LEFT constructs time tokens, frequency tokens from a differentiable STFT, and multi-scale structural tokens from a Nyquist-constrained learnable filterbank whose residual mask covers the spectrum above the learned structural bands. LEFT uses lightweight token interactions to fuse the three views and expose cross-view inconsistency. It also applies cross-scale reconstruction consistency and time–frequency cycle consistency, which makes the inconsistency signal more reliable. Experiments on real-world benchmarks show that LEFT improves detection quality and efficiency. Compared with strong baselines, LEFT improves VUS-ROC and VUS-PR by about 3\% and 6\%, respectively, while reducing FLOPs by over 80\%, speeding up training by about $8\times$. Future work will study when cross-view inconsistency is most informative and will improve interpretability by linking detected anomalies to specific time–frequency or cross-scale conflicts.

\section{Acknowledgments}
This work is supported by the National Natural Science Foundation of China (No. 12433011), the Guangdong Basic and Applied Basic Research Foundation (No. 2024A1515011962), the Australian Research Council (Nos. FT210100624, DP260100326, DE230101033, DP240101814, LP230200892, and LP240200546). 

\bibliographystyle{ACM-Reference-Format}
\bibliography{sample-base}

\appendix
\section{Appendices}

\renewcommand{\thefigure}{\thesection\arabic{figure}}  
\renewcommand{\thetable}{\thesection\arabic{table}}  
\setcounter{figure}{0}  
\setcounter{table}{0}

\subsection{Ablation Study Details}
\label{sec_ablation_appendix}
This appendix reports the ablation results in Table \ref{tab_ablation_results_appendix}.
LEFT achieves the best performance on each dataset and on average, reaching 0.7518 in VUS-ROC and 0.3590 in VUS-PR. Among single-component variants, cycle consistency brings the largest gain: the cycle-only setting in Row~3 raises the average to 0.6880 / 0.3320, while the other settings stay near the low baseline range. Row~13 confirms its importance, as removing cycle consistency with the other components retained sharply drops performance to 0.5422 / 0.2411. With cycle consistency, adding tri-view interaction or the learnable filterbank brings further gains, as Rows~9 and 7 show against Row~3, while Rows~11 and~12 remain below the full model. Cross-path consistency contributes beyond the cycle constraint, since removing it in Row~14 reduces the average from 0.7518 / 0.3590 to 0.7266 / 0.3416. Overall, time-frequency agreement is central to reliable learning, while the other components provide additional gains.

\begin{table}[t]
\centering
\caption{Detailed ablation studies. T.V.I., L.F.B., C.C., and C.P.C. denote tri-view interaction, learnable filterbank, cycle consistency, and cross-path consistency, respectively. For L.F.B., `\ding{51}'/`\ding{55}' denotes learnable/fixed filterbank.}
\resizebox{\linewidth}{!}{
\begin{tabular}{ccccc|cc|cc|cc|cc}
\hline
\multicolumn{5}{c|}{Dataset} & \multicolumn{2}{c|}{PSM} & \multicolumn{2}{c|}{MSL} & \multicolumn{2}{c|}{SMAP} & \multicolumn{2}{c}{avg.} \\
\hline
& T.V.I. & L.F.B. & C.C. & C.P.C. & V-R & V-P & V-R & V-P & V-R & V-P & V-R & V-P
\\\hline
1 & \ding{51} & \ding{55} & \ding{55} & \ding{55} & 0.6513 & 0.4438 & 0.5495 & 0.1849 & 0.4212 & 0.1138 & 0.5407 & 0.2475\\
2 & \ding{55} & \ding{51} & \ding{55} & \ding{55} & 0.6523 & 0.4447 & 0.5584 & 0.1869 & 0.4586 & 0.1202 & 0.5564 & 0.2506\\
3 & \ding{55} & \ding{55} & \ding{51} & \ding{55} & 0.7748 & 0.5491 & 0.7616 & 0.3090 & 0.5277 & 0.1378 & 0.6880 & 0.3320\\
4 & \ding{55} & \ding{55} & \ding{55} & \ding{51} & 0.6497 & 0.4415 & 0.5226 & 0.1572 & 0.4727 & 0.1269 & 0.5483 & 0.2419\\
5 & \ding{55} & \ding{55} & \ding{51} & \ding{51} & 0.7819 & 0.5684 & 0.7424 & 0.3075 & 0.4773 & 0.1275 & 0.6672 & 0.3344\\
6 & \ding{55} & \ding{51} & \ding{55} & \ding{51} & 0.6498 & 0.4412 & 0.5558 & 0.1896 & 0.4371 & 0.1182 & 0.5476 & 0.2497\\
7 & \ding{55} & \ding{51} & \ding{51} & \ding{55} & 0.7816 & 0.5582 & 0.8146 & 0.3291 & 0.5521 & 0.1424 & 0.7161 & 0.3432\\
8 & \ding{51} & \ding{55} & \ding{55} & \ding{51} & 0.6490 & 0.4416 & 0.5391 & 0.1779 & 0.4398 & 0.1172 & 0.5427 & 0.2456\\
9 & \ding{51} & \ding{55} & \ding{51} & \ding{55} & 0.7795 & 0.5614 & 0.8124 & 0.3316 & 0.5922 & 0.1556 & 0.7280 & 0.3495\\
10 & \ding{51} & \ding{51} & \ding{55} & \ding{55} & 0.6527 & 0.4440 & 0.5448 & 0.1811 & 0.4848 & 0.1277 & 0.5608 & 0.2510\\
11 & \ding{55} & \ding{51} & \ding{51} & \ding{51} & 0.7814 & 0.5653 & 0.8142 & 0.3235 & 0.5814 & 0.1477 & 0.7256 & 0.3455\\
12 & \ding{51} & \ding{55} & \ding{51} & \ding{51} & 0.7748 & 0.5564 & 0.8082 & 0.3254 & 0.5952 & 0.1549 & 0.7261 & 0.3456\\
13 & \ding{51} & \ding{51} & \ding{55} & \ding{51} & 0.6497 & 0.4407 & 0.5283 & 0.1666 & 0.4487 & 0.1161 & 0.5422 & 0.2411\\
14 & \ding{51} & \ding{51} & \ding{51} & \ding{55} & 0.7753 & 0.5516 & 0.8100 & 0.3251 & 0.5946 & 0.1480 & 0.7266 & 0.3416\\

LEFT & \ding{51} & \ding{51} & \ding{51} & \ding{51}  & \textbf{0.7836} & \textbf{0.5701} & \textbf{0.8157} & \textbf{0.3342} & \textbf{0.6562} & \textbf{0.1726} & \textbf{0.7518} & \textbf{0.3590}\\
\hline

\end{tabular}
}
\label{tab_ablation_results_appendix}
\end{table}

\begin{table}[t]
\centering
\caption{Fusion strategy ablation. 
}
\resizebox{\linewidth}{!}{
\begin{tabular}{ccccrr}
\hline
\multirow{2}{*}{Dataset} & \multirow{2}{*}{Fusion Method} & \multirow{2}{*}{VUS-ROC} & \multirow{2}{*}{VUS-PR} & \multicolumn{1}{c}{Training } & \multicolumn{1}{c}{Inference}\\
& & & & \multicolumn{1}{c}{Time (s / epoch)} & \multicolumn{1}{c}{Time (s / batch)}\\
\hline
\multirow{5}{*}{PSM} 
& $M\leftrightarrow F\leftrightarrow T$ & \textbf{0.7860} & 0.5692 & 30.29 & 0.01699\\
& $M\leftrightarrow F$ & 0.7687 & 0.5479 & 20.56 & 0.01080\\
& $M\leftrightarrow T$ & 0.7558 & 0.5371 & 18.93 & 0.01060\\
& $T\leftrightarrow F$ & 0.7790 & 0.5593 & 18.79 & 0.01066\\
& Default & 0.7836 & \textbf{0.5701} & 25.66 & 0.01527\\
\hline
\multirow{5}{*}{MSL} 
& $M\leftrightarrow F\leftrightarrow T$ & 0.8106 & 0.3246 & 7.65 & 0.00854\\
& $M\leftrightarrow F$ & 0.8069 & 0.3255 & 6.62 & 0.00788\\
& $M\leftrightarrow T$ & 0.8051 & 0.3149 & 6.31 & 0.00787\\
& $T\leftrightarrow F$ & 0.8121 & 0.3265 & 6.46 & 0.00782\\
& Default & \textbf{0.8157} & \textbf{0.3342} & 7.02 & 0.00841\\
\hline
\multirow{5}{*}{SMAP}
& $M\leftrightarrow F\leftrightarrow T$ & 0.4776 & 0.1295 & 17.71 & 0.00854\\
& $M\leftrightarrow F$ & 0.5640 & 0.1419 & 15.40 & 0.00790\\
& $M\leftrightarrow T$ & 0.4330 & 0.1163 & 14.62 & 0.00788\\
& $T\leftrightarrow F$ & 0.6203 & 0.1634 & 15.50 & 0.00854\\ 
& Default & \textbf{0.6562} & \textbf{0.1726} & 16.34 & 0.00856\\
\hline
\end{tabular}
}
\label{tab_fusion_results_appendix}
\end{table}

\begin{table}[t]
\centering
\caption{JS and WD comparison.}
\resizebox{.8\linewidth}{!}{
\begin{tabular}{c|cc|cc|cc}
\hline
\multirow{2}{*}{} & \multicolumn{2}{c|}{PSM} & \multicolumn{2}{c|}{MSL} & \multicolumn{2}{c}{SMAP} \\
& V-R & V-P & V-R & V-P & V-R & V-P 
\\\hline
LEFT \textit{w} WD & 0.7734 & 0.5577 & 0.8128 & 0.3321 & 0.6123 & 0.1579\\
LEFT \textit{w} JS & \textbf{0.7836} & \textbf{0.5701} & \textbf{0.8157} & \textbf{0.3342} & \textbf{0.6562} & \textbf{0.1726}\\
\hline
\end{tabular}
}
\label{tab_js_wd}
\end{table}

\subsection{Fusion Strategy Ablation}
\label{sec_fusion_results_appendix}
Table~\ref{tab_fusion_results_appendix} compares fusion strategies on PSM, MSL, and SMAP using VUS-ROC, VUS-PR, and measured runtime. The proposed strategy gives the most consistent detection quality, achieving the best VUS-PR on PSM and the best VUS-ROC and VUS-PR on MSL and SMAP. Although all-pairs fusion $M \leftrightarrow F \leftrightarrow T$ obtains the highest VUS-ROC on PSM, it has higher compute cost and drops markedly on SMAP, indicating that denser interaction does not provide stable gains. The $T \leftrightarrow F$ variant remains competitive on PSM and MSL and clearly outperforms $M \leftrightarrow F$ and $M \leftrightarrow T$ on SMAP, suggesting the importance of direct time-frequency exchange. In contrast, strategies mixing $M$ with only one branch show weaker ranking quality, despite runtimes similar to the proposed strategy.

\subsection{Fusion Hyper-parameters Sensitivity} 
\label{sec_cyc_cons_appendix}
\begin{figure*}[t]
\centering
\includegraphics[width=0.33\columnwidth]{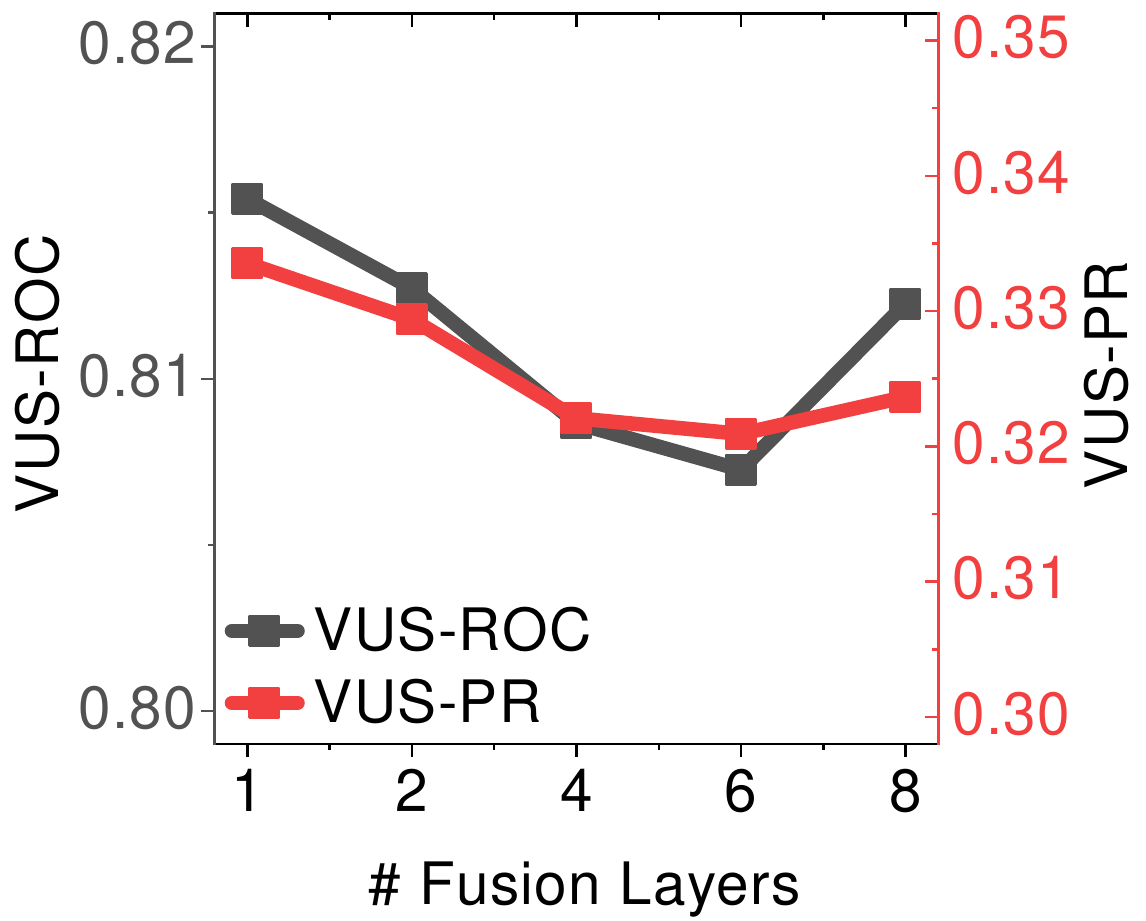}
\llap{(a.1)~}
\includegraphics[width=0.33\columnwidth]{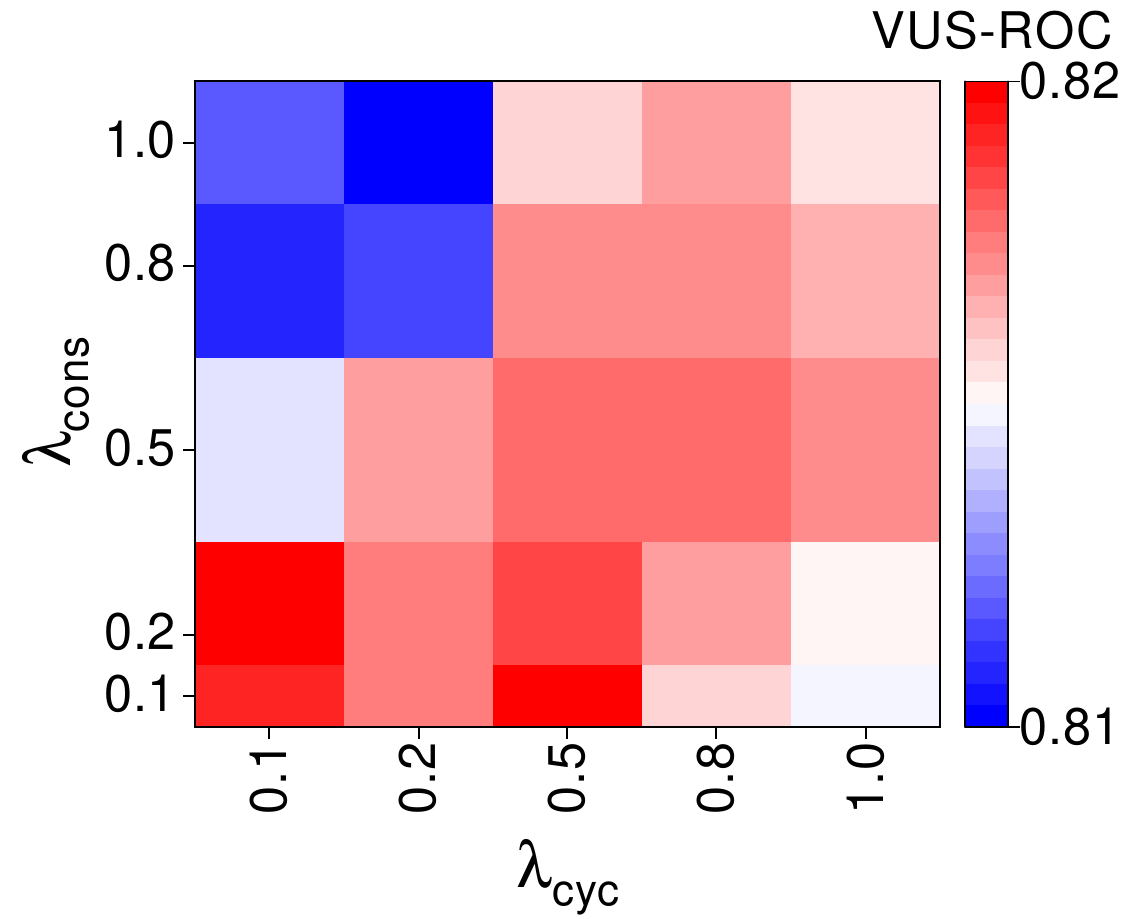}
\llap{(a.3)~}
\includegraphics[width=0.33\columnwidth]{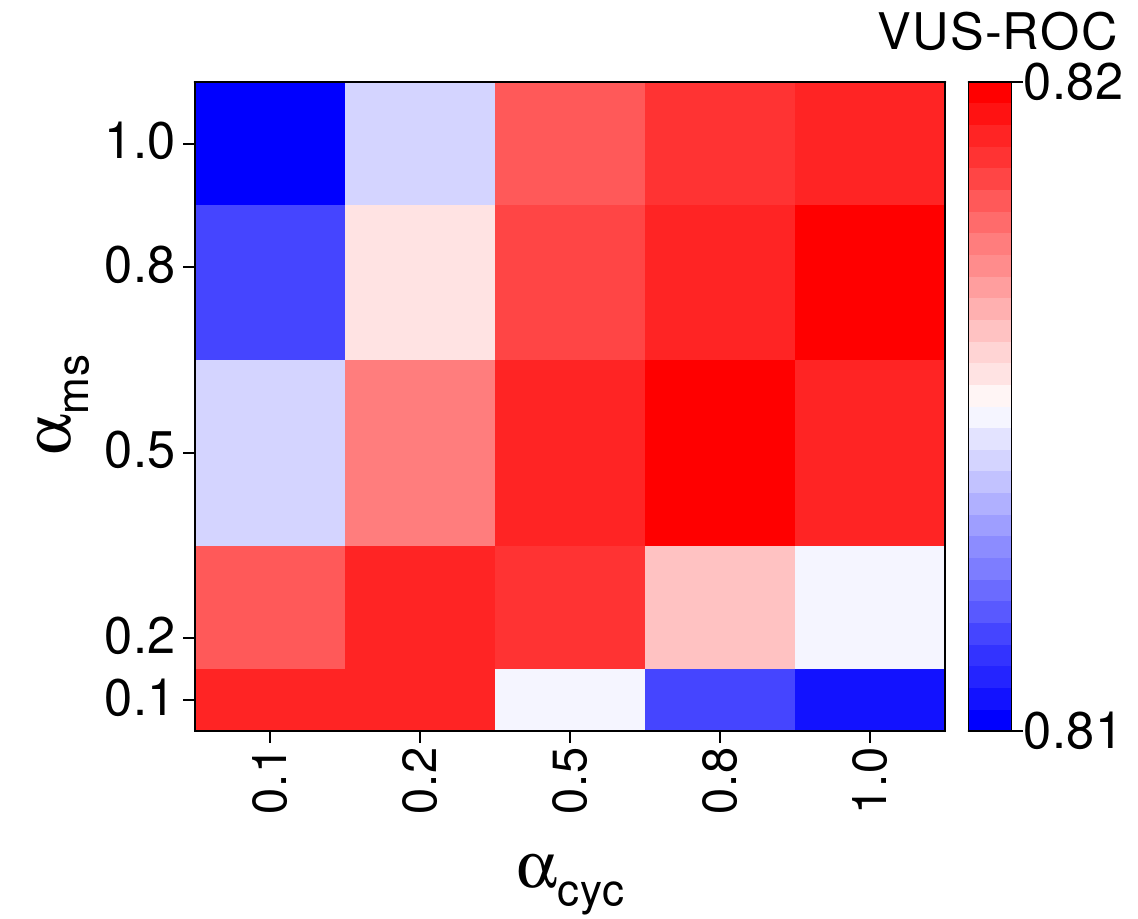}
\llap{(a.5)~}
\includegraphics[width=0.33\columnwidth]{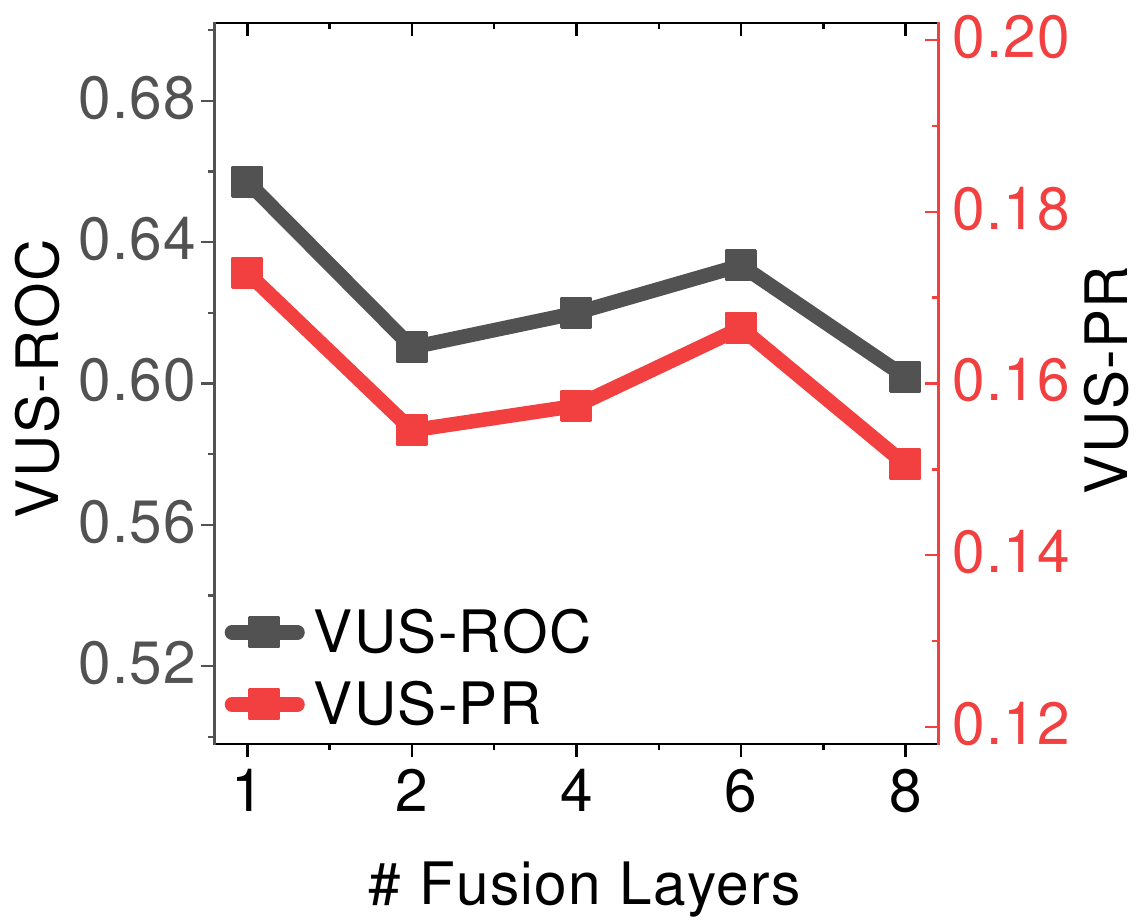}
\llap{(b.1)~}
\includegraphics[width=0.33\columnwidth]{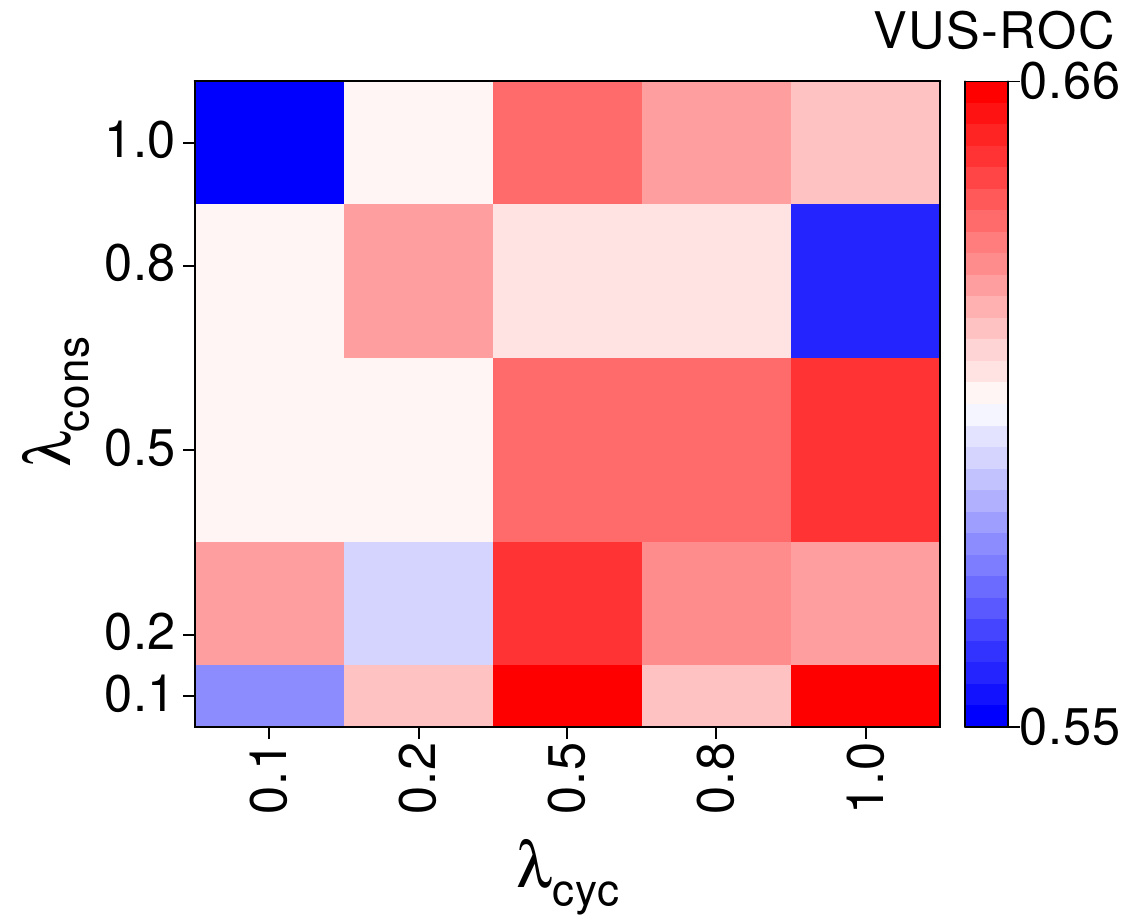}
\llap{(b.3)~}
\includegraphics[width=0.33\columnwidth]{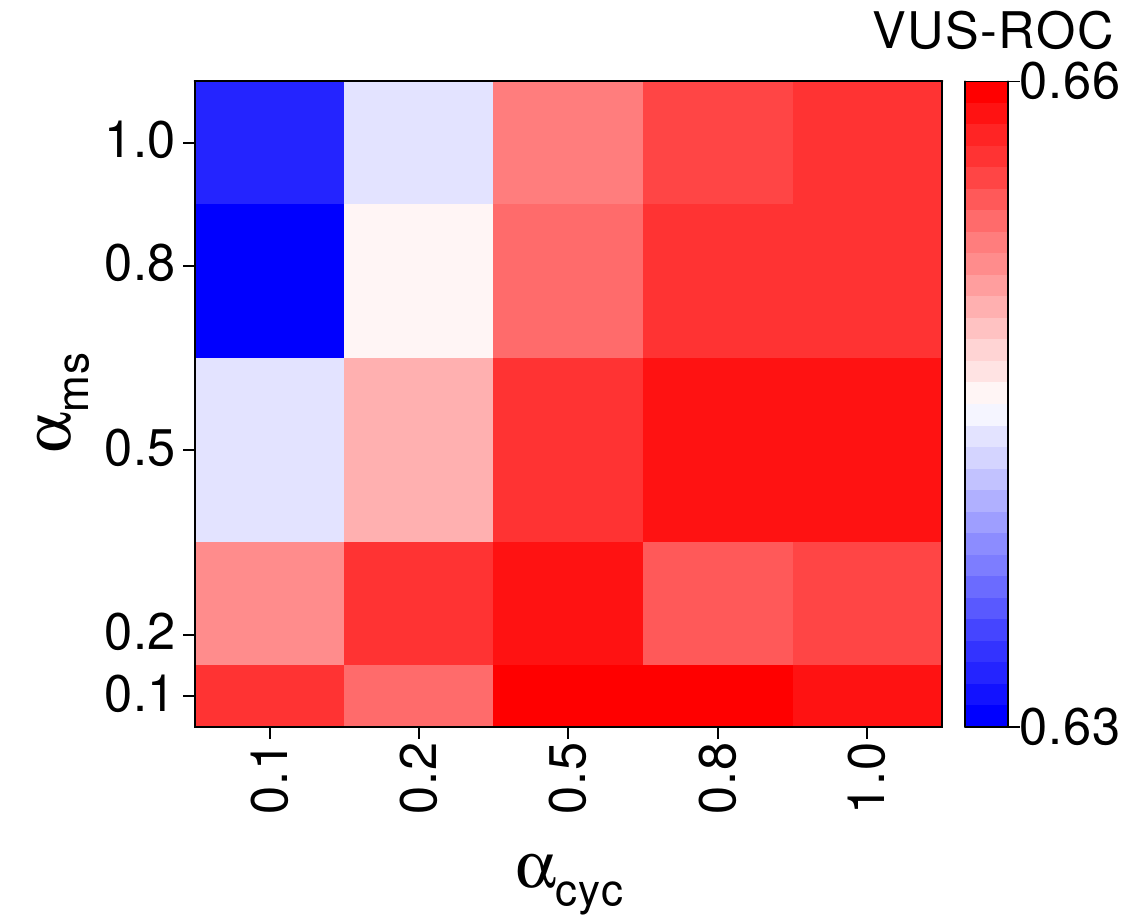}
\llap{(b.5)~}

\includegraphics[width=0.33\columnwidth]{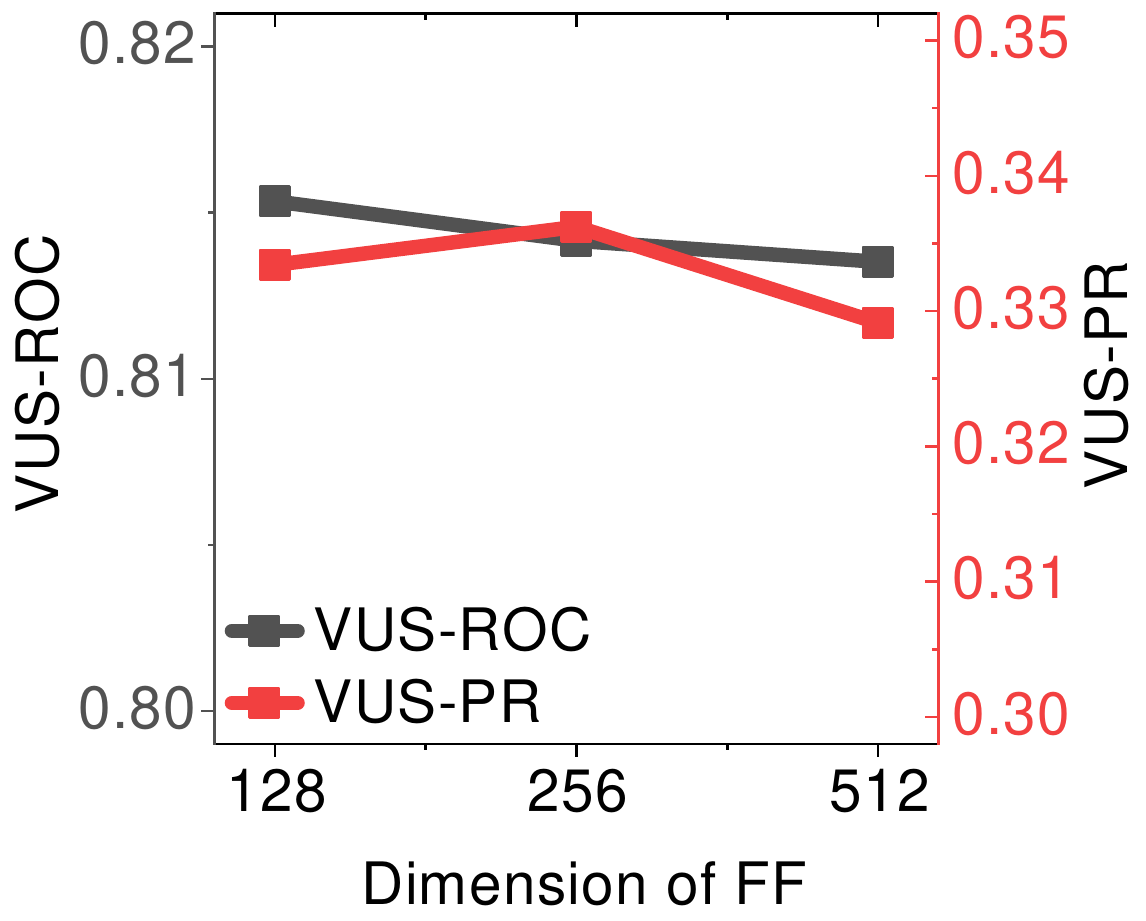}
\llap{(a.2)~}
\includegraphics[width=0.33\columnwidth]{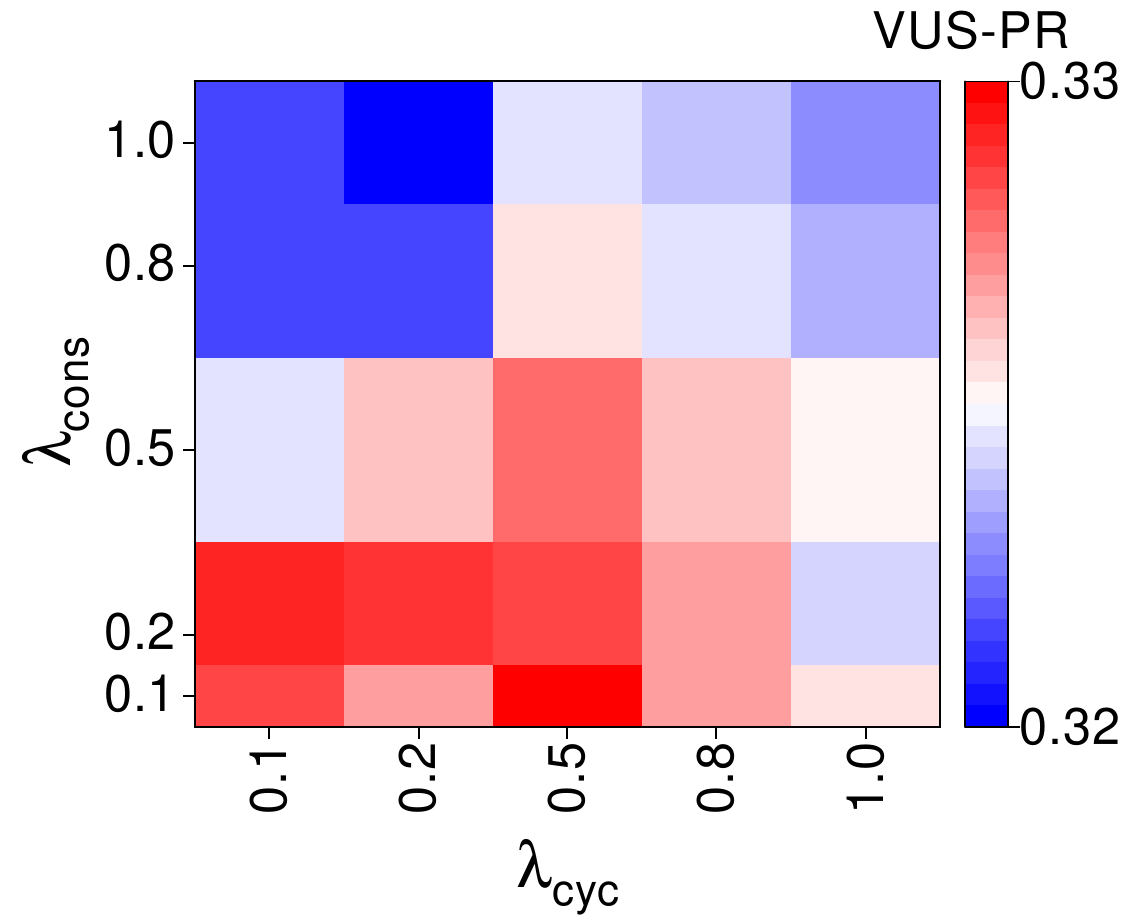}
\llap{(a.4)~}
\includegraphics[width=0.33\columnwidth]{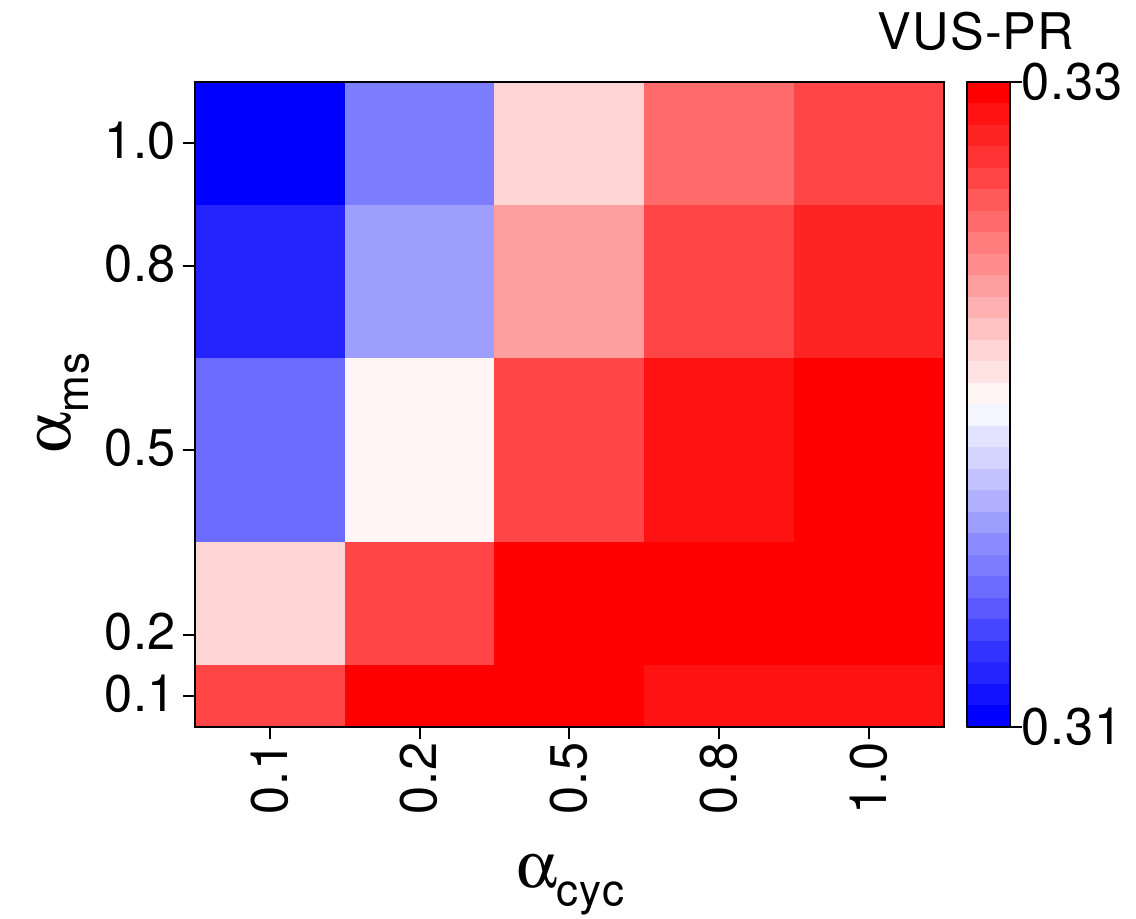}
\llap{(a.6)~}
\includegraphics[width=0.33\columnwidth]{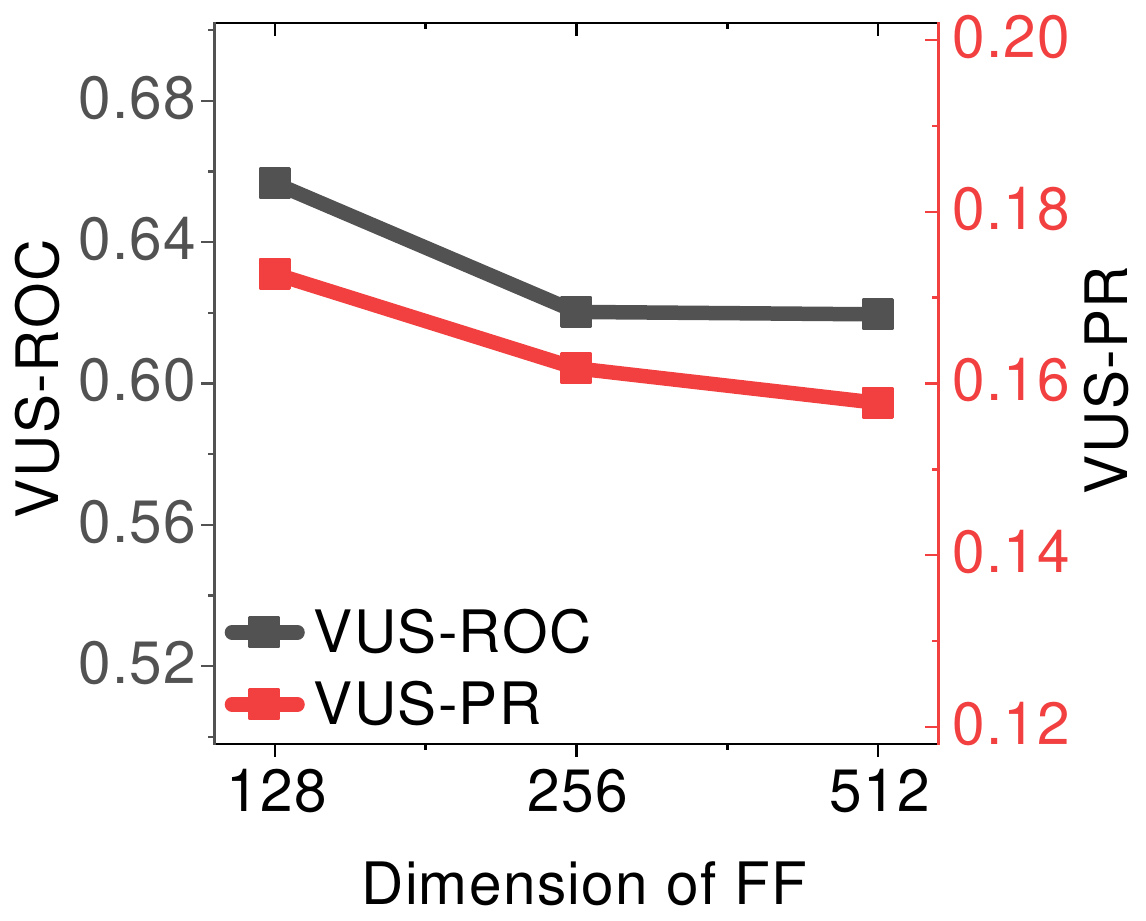}
\llap{(b.2)~}
\includegraphics[width=0.33\columnwidth]{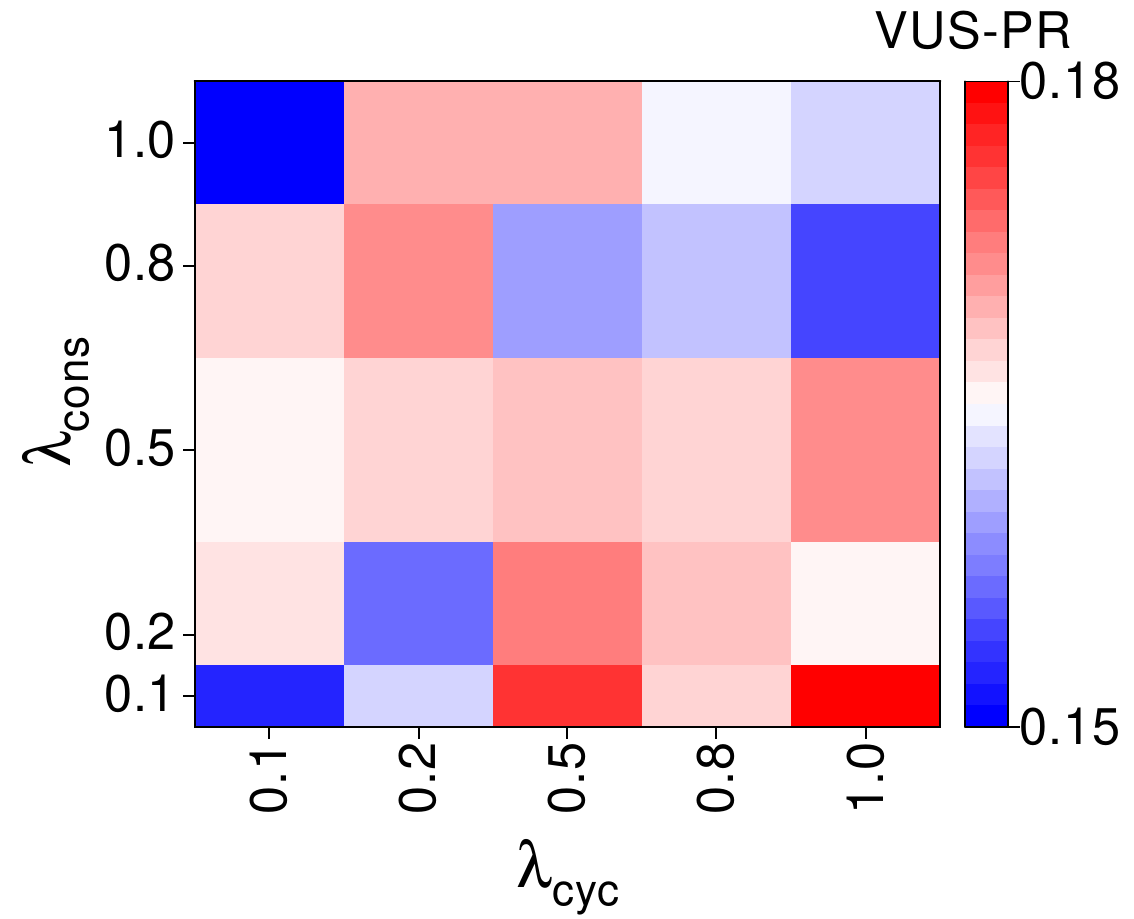}
\llap{(b.4)~}
\includegraphics[width=0.33\columnwidth]{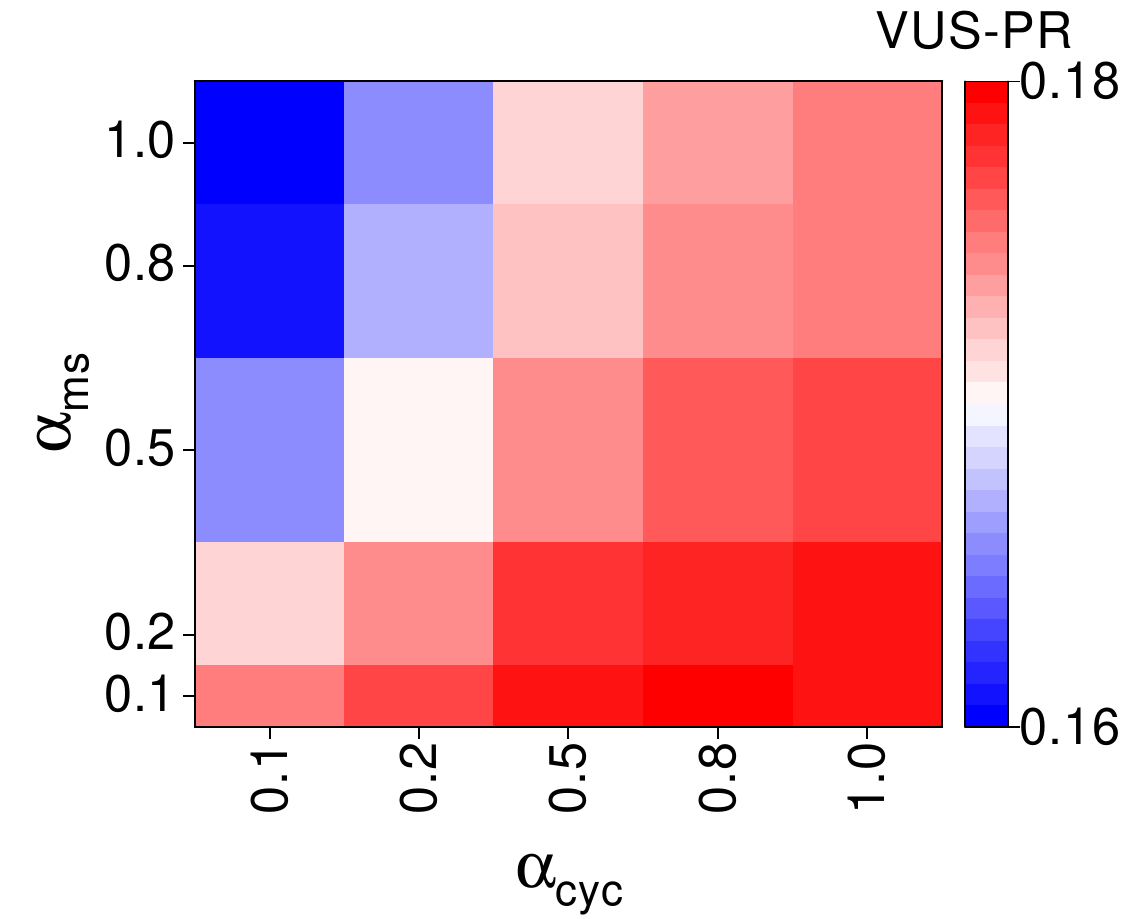}
\llap{(b.6)~}
\caption{Fusion hyper-parameters sensitivity analysis on the MSL (a.x) and SMAP (b.x). (x.1) and (x.2) present the depth and width analysis. (x.3) and (x.4) show the sensitivity to $\lambda_{\mathrm{cyc}}$ and $\lambda_{\mathrm{cons}}$. (x.5) and (x.6): $\alpha_{\mathrm{cyc}}$ and  $\alpha_{\mathrm{ms}}$ sensitivity.}
\label{fig_fusion_hyper_msl}
\end{figure*}

Figs.~\ref{fig_fusion_hyper_msl}(a.1, a.2, b.1, b.2) show that increasing fusion depth or width brings no clear gain, and several settings slightly reduce VUS-ROC or VUS-PR, consistent with the observation that stronger fusion is not necessarily beneficial. Figs.~\ref{fig_fusion_hyper_msl}(a.3, a.4, b.3, b.4) further show that performance depends on the joint setting of cycle consistency and cross-path consistency. Higher scores appear when cross-path consistency remains low and cycle consistency is moderate, while large cross-path consistency does not yield reliable improvements. Figs.~\ref{fig_fusion_hyper_msl}(a.5, a.6, b.5, b.6) show that VUS-ROC and VUS-PR vary with the joint choice of $\alpha_{\mathrm{cyc}}$ and $\alpha_{\mathrm{ms}}$. The trend is not monotonic along either axis, since increasing one weight does not reliably improve the score when the other is fixed. Higher values concentrate in a middle region of the grid, where $\alpha_{\mathrm{cyc}}$ is moderate to large and $\alpha_{\mathrm{ms}}$ stays below its maximum, while several boundary settings yield lower scores. This pattern appears in both metrics, which supports the scoring design that balances cycle-based disagreement and multi-scale residual evidence rather than letting one term dominate.

\subsection{Sensitivity of Downsampling Factors} 
\label{sec_multi_scale_appendix}
\begin{figure}[t]
\centering
\includegraphics[width=.31\columnwidth]{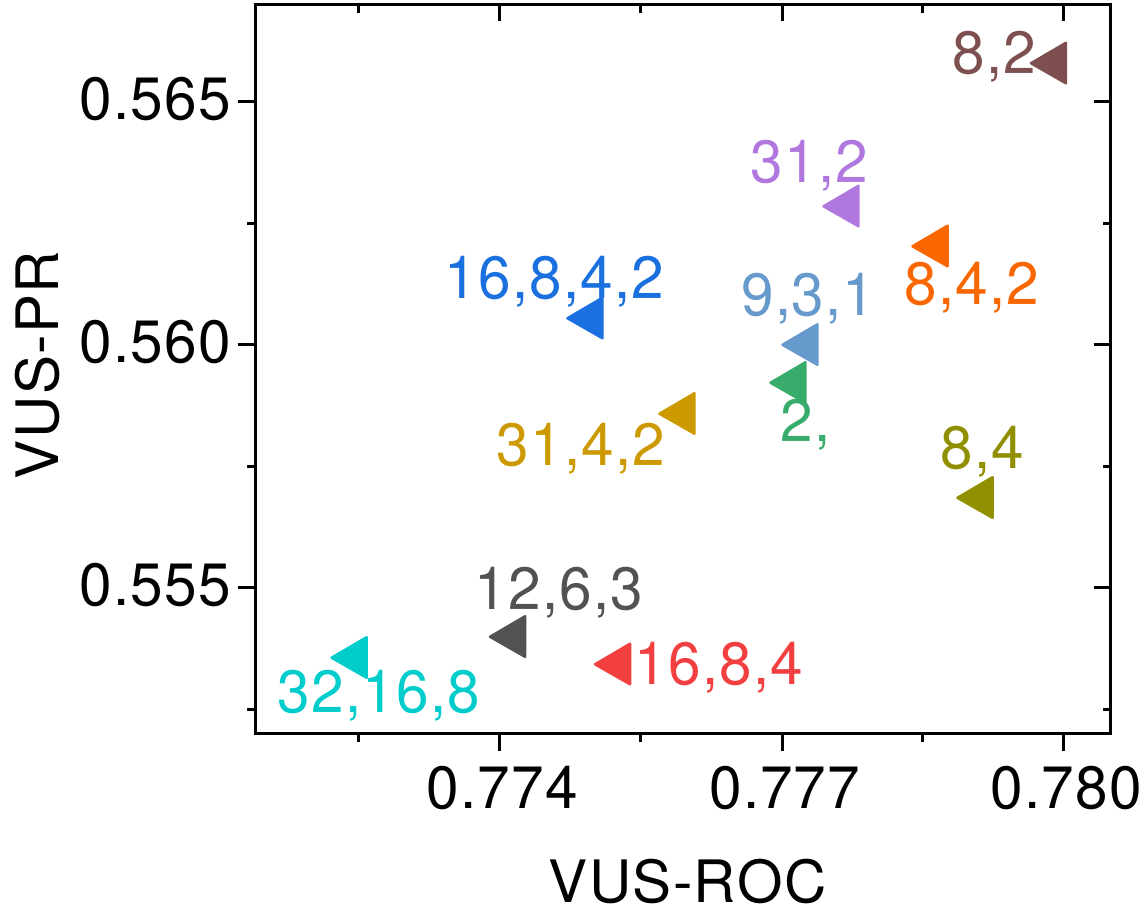}
\llap{(a)~}
\includegraphics[width=.31\columnwidth]{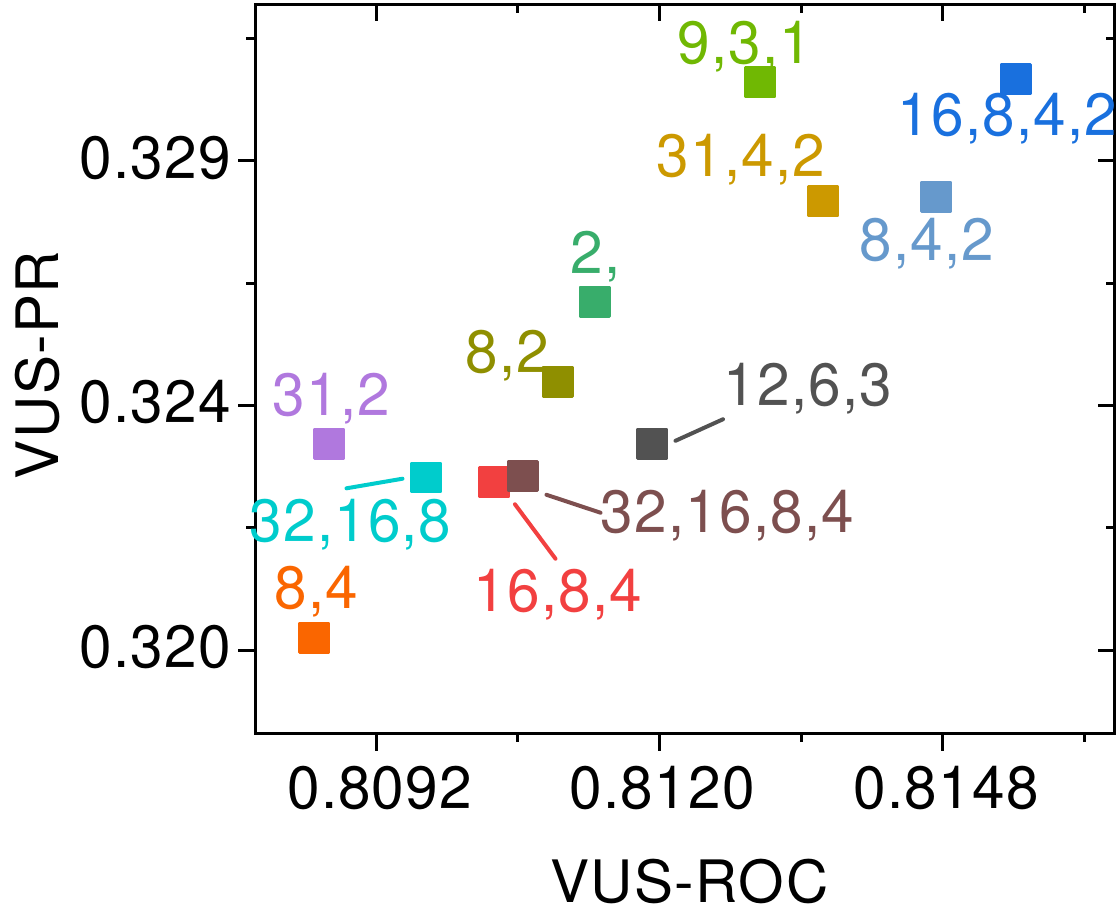}
\llap{(b)~}
\includegraphics[width=.31\columnwidth]{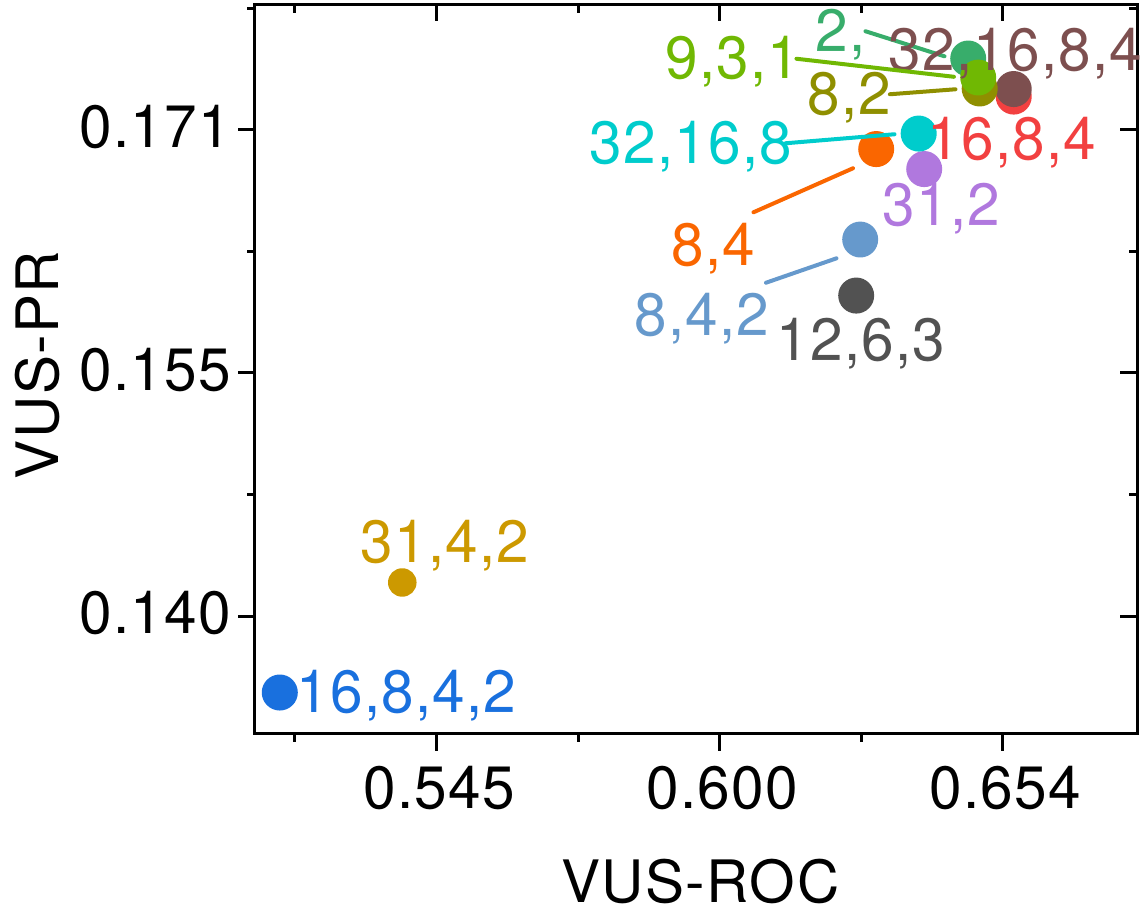}
\llap{(c)~}
\caption{Sensitivity of $\{r_k\}_{k=1}^K$. (a): PSM. (b): MSL. (c): SMAP. }
\label{fig_multi_scale}
\end{figure}
Fig.~\ref{fig_multi_scale} shows that downsampling factors $\{r_k\}_{k=1}^K$ affect performance. The VUS-ROC and VUS-PR results in Table~\ref{tab_main_results} use fixed factors, such as (16, 8, 4) for SMAP, which may not be optimal. A more thorough hyper-parameter search may further improve performance.

\subsection{Effect of JS vs. Wasserstein}

In LEFT, JS divergence serves as a prototype calibration criterion for reliable EMA updates. To examine this choice, we replace the JS discrepancy with a Wasserstein-based variant (WD), keeping the other components unchanged. As shown in Table~\ref{tab_js_wd}, WD does not improve performance on PSM, MSL, or SMAP.

\subsection{Limitations}
This work develops a unified tri-view framework for unsupervised TSAD and shows competitive performance with favorable efficiency on the evaluated benchmarks. Nevertheless, this study is validated on a limited set of datasets and has not been tested on broader benchmarks such as UCR \cite{wu_current_2021} or in practical deployment settings. In addition, the method may still produce false alarms in some cases, which remains a direction for future work.

\subsection{Details of the Proof}
\label{sec_proof_appendix}
\subsubsection*{Proof of Lemma \ref{lem_nyquist_alias}}
\label{sec_proof_lemma1}
\begin{proof}
Since $r_1\ge\cdots\ge r_K$, the corresponding Nyquist cutoffs satisfy $c_1\le\cdots\le c_K$. Since $\sigma(u_k)\in(0,1)$, we have
$e_k-e_{k-1}=(c_k-e_{k-1})\sigma(u_k)\ge 0$,
which shows that the edge sequence $\{e_k\}$ is monotone nondecreasing. Moreover,
$e_k=e_{k-1}+(c_k-e_{k-1})\sigma(u_k)\le e_{k-1}+(c_k-e_{k-1})=c_k$,
so $e_k\le c_k$ for all $k$.
By the definition of $\tilde{\bm{S}}^{(k)}(f)$, downsampling periodically replicates the spectrum and may introduce overlap among replicas. Therefore, the aliasing term is dominated by the residual spectral energy above the cutoff frequency $c_k$, up to a scaling factor $\vartheta(r_k)$. Accordingly,
\begin{equation}
\|\Delta_{r_k}(\bm{X}^{(k)})\|_{2}
\le
\vartheta(r_k)\sum_{f>c_k}\|\tilde{\bm{S}}^{(k)}(f)\|_{2}.
\label{eq-lemma1-proof-1}
\end{equation}

Using $\tilde{\bm{S}}^{(k)}(f)=\tilde m_k(f)\tilde{\bm{S}}(f)$ and $0\le \tilde m_k(f)\le 1$, we obtain,
\begin{equation}
\begin{split}
\sum_{f>c_k}\|\tilde{\bm{S}}^{(k)}(f)\|_{2}
&=
\sum_{f>c_k}\tilde m_k(f)\|\tilde{\bm{S}}(f)\|_{2}\\
&\le
\Bigl(\max_{f>c_k}\|\tilde{\bm{S}}(f)\|_{2}\Bigr)\sum_{f>c_k}\tilde m_k(f)\le
\|\tilde{\bm{X}}\|_{2}\,\epsilon_k,
\end{split}
\end{equation}
where the last inequality uses $\max_f \|\tilde{\bm{S}}(f)\|_{2}\le \|\tilde{\bm{X}}\|_{2}$. Substituting this bound into the previous inequality~\eqref{eq-lemma1-proof-1} yields~\eqref{eq_alias_bound}.
\end{proof}

\subsubsection*{Proof of Lemma \ref{lem_tf_cycle}}
\label{sec_proof_lemma2}

\begin{proof}
By definition,
\begin{equation}
\hat{\bm{S}}_{\rightarrow f}-\bm{S}
= W_\theta(\hat{\bm{X}})-W_\theta(\bm{X})
= W_\theta(\hat{\bm{X}}-\bm{X}).   
\end{equation}

Applying~\eqref{eq_frame_bounds} to $\bm{Y}=\hat{\bm{X}}-\bm{X}$ gives,
\begin{equation}
\sqrt{A}\,\|\hat{\bm{X}}-\bm{X}\|_2 \le \|\hat{\bm{S}}_{\rightarrow f}-\bm{S}\|_2
\le \sqrt{B}\,\|\hat{\bm{X}}-\bm{X}\|_2,
\end{equation}
which yields~\eqref{eq_time_from_freq} and~\eqref{eq_freq_from_time}.
\end{proof}

\subsubsection*{Proof of Lemma \ref{lem_score_closed_loop}}
\label{sec_proof_lemma3}

\begin{proof}
We repeatedly use the elementary inequality for SmoothL1:
for any scalar residual $r$,
\begin{equation}
|r|\le 2\,l(r,0)+1,
\label{eq_smoothl1_inline}
\end{equation}
which implies the same inequality after averaging over entries.
$\mathcal{A}_{\mathrm{ms}}(t)$ is a weighted sum of absolute reconstruction residuals aggregated across supervised scales.
Applying~\eqref{eq_smoothl1_inline} element-wise to each supervised residual and absorbing fixed dimensional factors yields,
\begin{equation}
\langle \mathcal{A}_{\mathrm{ms}}(t)\rangle_t
\ \le\
\kappa_{\mathrm{ms}}\mathcal{L}_{\mathrm{ms}}+\kappa_0^{(1)}.
\end{equation}

By definition, $\mathcal{A}_{\mathrm{cyc}}(t)=\bigl(\mathrm{MA}_\kappa(\tilde{\mathcal{A}}_{\mathrm{cyc}})\bigr)(t)$
with $\tilde{\mathcal{A}}_{\mathrm{cyc}}(t)\ge 0$.
The term $|\hat{\bm{X}}_{\leftarrow f}-\bm{X}|$ is directly controlled by the time domain SmoothL1 term
$l(\hat{\bm{X}}_{\leftarrow f},\bm{X})$ in $\mathcal{L}_{\mathrm{cyc}}$ via~\eqref{eq_smoothl1_inline}.
For $|\hat{\bm{X}}-\bm{X}|$, Lemma~\ref{lem_tf_cycle} gives
$\|\hat{\bm{X}}-\bm{X}\|_2 \le \tfrac{1}{\sqrt{A}}\|\hat{\bm{S}}_{\rightarrow f}-\bm{S}\|_2$.
Standard norm relations convert this to a bound on the mean absolute residual up to a fixed constant, and
$\|\hat{\bm{S}}_{\rightarrow f}-\bm{S}\|_2$ is controlled by the spectral SmoothL1 term
$l(\hat{\bm{S}}_{\rightarrow f},\bm{S})$ in $\mathcal{L}_{\mathrm{cyc}}$ via~\eqref{eq_smoothl1_inline}.
Hence,
\begin{equation}
\alpha_f\langle |\hat{\bm{X}}_{\leftarrow f}-\bm{X}|\rangle_t
+
\alpha_t\langle |\hat{\bm{X}}-\bm{X}|\rangle_t
\ \le\
\kappa_{\mathrm{cyc}}\mathcal{L}_{\mathrm{cyc}}+\kappa_0^{(2)}.
\end{equation}

For the cross-path term, using the definition $c(t)=|\hat{\bm{X}}_{\mathrm{ms}}(t)-\hat{\bm{X}}_{\leftarrow f}(t)|$
and applying~\eqref{eq_smoothl1_inline} to $l(\hat{\bm{X}}_{\mathrm{ms}},\hat{\bm{X}}_{\leftarrow f})=\mathcal{L}_{\mathrm{cons}}$ gives,
\begin{equation}
\alpha_c\langle c(t)\rangle_t
\ \le\
\kappa_{\mathrm{cons}}\mathcal{L}_{\mathrm{cons}}+\kappa_0^{(3)}.
\end{equation}

Combining the above with the nonnegative gate term yields,
\begin{equation}
\bigl\langle \tilde{\mathcal{A}}_{\mathrm{cyc}}(t)\bigr\rangle_t
\ \le\
\kappa_{\mathrm{cyc}}\mathcal{L}_{\mathrm{cyc}}
+
\kappa_{\mathrm{cons}}\mathcal{L}_{\mathrm{cons}}
+
\kappa_g\,\alpha_g\,\langle g(t)\rangle_t
+
\kappa_0^{(4)}.
\end{equation}
Applying~\eqref{eq_ma_rho} with $u(t)=\tilde{\mathcal{A}}_{\mathrm{cyc}}(t)$ yields,
\begin{equation}
\langle \mathcal{A}_{\mathrm{cyc}}(t)\rangle_t
\ \le\
\rho_\kappa\Bigl(
\kappa_{\mathrm{cyc}}\mathcal{L}_{\mathrm{cyc}}
+
\kappa_{\mathrm{cons}}\mathcal{L}_{\mathrm{cons}}
+
\kappa_g\,\alpha_g\,\langle g(t)\rangle_t
+
\kappa_0^{(4)}
\Bigr).
\end{equation}
Using $\mathcal{A}(t)=\alpha_{\mathrm{cyc}}\mathcal{A}_{\mathrm{cyc}}(t)+\alpha_{\mathrm{ms}}\mathcal{A}_{\mathrm{ms}}(t)$,
absorbing fixed coefficients into $\kappa_{\mathrm{ms}},\kappa_{\mathrm{cyc}},\kappa_{\mathrm{cons}}$,
and substituting $\mathcal{L}_{\mathrm{ms}}\le\varepsilon_{\mathrm{ms}}$,
$\mathcal{L}_{\mathrm{cyc}}\le\varepsilon_{\mathrm{cyc}}$,
$\mathcal{L}_{\mathrm{cons}}\le\varepsilon_{\mathrm{cons}}$
gives~\eqref{eq_normal_upper}.
Under the stated conditions on $\Omega$ and nonnegativity of all terms,
\begin{equation}
\begin{aligned}
&\bigl\langle \tilde{\mathcal{A}}_{\mathrm{cyc}}(t)\bigr\rangle_{t\in\Omega}
\\
&\ge
\alpha_t\bigl\langle |\hat{\bm{X}}(t)-\bm{X}(t)|\bigr\rangle_{t\in\Omega}
+
\alpha_f\bigl\langle |\hat{\bm{X}}_{\leftarrow f}(t)-\bm{X}(t)|\bigr\rangle_{t\in\Omega}
+
\alpha_c\langle c(t)\rangle_{t\in\Omega}
\\
&\ge
\alpha_t\delta_t+\alpha_f\delta_f+\alpha_c\delta_c,
\end{aligned}    
\end{equation}
which yields~\eqref{eq_anom_lower_cycle}.
If the pointwise bounds hold and the averaging window lies in $\Omega$, then
$\mathcal{A}_{\mathrm{cyc}}(t)=\mathrm{MA}_\kappa(\tilde{\mathcal{A}}_{\mathrm{cyc}})(t)$
preserves the same lower bound.
\end{proof}

\end{document}